\documentclass{article}

\usepackage[preprint]{neurips_2026}

\usepackage[utf8]{inputenc} 
\usepackage[T1]{fontenc}    
\usepackage{hyperref}       
\usepackage{url}            
\usepackage{booktabs}       
\usepackage{multirow}       
\usepackage{amsfonts}       
\usepackage{nicefrac}       
\usepackage{microtype}      
\usepackage{xcolor}         
\usepackage{amsmath}
\usepackage{amssymb}
\usepackage{amsthm}
\usepackage{graphicx}
\usepackage{subcaption} 
\usepackage{caption}    
\usepackage{bbm}
\usepackage{soul}
\usepackage{xcolor}
\usepackage{float} 

\newtheorem{lemma}{Lemma}
\newtheorem{proposition}{Proposition}
\newtheorem{corollary}{Corollary}

\usepackage[most]{tcolorbox}

\definecolor{takeawaybg}{HTML}{C6D3E7}

\newtcolorbox{takeawaybox}{
  colback=takeawaybg,
  colframe=takeawaybg,
  boxrule=0pt,
  arc=2pt,
  left=8pt,
  right=8pt,
  top=6pt,
  bottom=6pt
}

\title{LpWM: A Case for Sparse Representations in World Models}

\author{
  Yilun Kuang\thanks{Correspondence to: Yilun Kuang <yilun.kuang@nyu.edu>, Yann LeCun <yann.lecun@nyu.edu>. \noindent\textbf{Code:} \url{https://github.com/YilunKuang/lpworldmodel}.} \\
  NYU, AMI Labs\\
  \And
  Yash Dagade \\
  Duke University \\
  \And
  Quentin Le Lidec \\
  NYU, AMI Labs \\
  \And
  Lucas Maes \\
  Mila, AMI Labs \\
  \AND
  Randall Balestriero \\
  Brown University, AMI Labs \\
  \And
  Yann LeCun \\
  NYU, AMI Labs \\
}

\begin{document}

\maketitle

\begin{abstract}
Joint-embedding predictive architectures (JEPAs) learn latent dynamics for planning and avoid representation collapse by matching features to maximum-entropy distributions such as isotropic Gaussians, yielding dense representations. However, it is unclear whether dense representations are the most favorable geometry for modeling dynamics. In this work, we ask whether a different geometry—sparse representations—can make action-conditioned latent dynamics easier to model, and what dynamical structure emerges from such representations. We first show that nonlinear Lipschitz dynamics can be approximated arbitrarily well by action-conditioned linear dynamics in a sufficiently high-dimensional one-hot latent space, with rollout error vanishing as the dimension grows. This motivates distributed sparse representations as a practical relaxation of one-hot sparsity. We introduce LpWorldModel (LpWM), a JEPA model regularized with Rectified Distribution Matching Regularization (RDMReg) to match encoder features to a Rectified Generalized Gaussian distribution, yielding non-negative sparse codes. Empirically, sparsity lowers the predictor complexity required for successful planning: on PushT, sparse LpWM outperforms dense LeWM by up to 57\% in planning success at intermediate predictor capacities. This advantage also extends beyond Gaussian distribution matching, with LpWM outperforming dense VICReg representations across multiple predictor families. We further find that the learned sparse representations are mode-factored, with support encoding discrete dynamical regimes and feature magnitudes capturing continuous within-regime state. Together, these results suggest that sparse representations can reduce the predictor complexity required for control while revealing interpretable structure.
\end{abstract}

\section{Introduction}
\label{sec:intro}

Joint-embedding predictive architectures (JEPA) have emerged as a practical recipe for model-based control from pixels: an encoder maps observations to a latent space, a predictor rollouts the latent forward under actions, and planning is carried out by model-predictive control (MPC) in latent space \citep{lecun2022path}. In the case of end-to-end training, where the encoder and predictor are jointly trained, an anti-collapse loss is necessary to prevent the feature representation from degenerating to a constant or low-rank subspaces \citep{jing2021understanding}. 

Prior methods such as PLDM \citep{sobal2025learningrewardfreeofflinedata} use VCReg \citep{bardes2022vicregvarianceinvariancecovarianceregularizationselfsupervised}, while LeWM \citep{maes2026leworldmodelstableendtoendjointembedding} uses SIGReg \citep{balestriero2025lejepaprovablescalableselfsupervised} to prevent collapse through information maximization, yielding dense representations in which all latent coordinates are nonzero almost surely. While effective for preventing collapse, it remains unclear whether dense features are the most favorable geometry for modeling controlled dynamics.

In this work, we ask whether sparse representations provide a more favorable geometry for modeling action-conditioned dynamics. Our motivation comes from an idealized limit: for Lipschitz controlled dynamics on a compact state space, sufficiently high-dimensional one-hot representations can approximate the dynamics with action-conditioned linear latent dynamics (Proposition~\ref{prop:linearapproxofnonlinear}), with rollout error vanishing as the dimension grows (Corollary~\ref{coro:sparsityinhighd}). While exact one-hot codes are impractical, this suggests distributed sparse representations, in which a nontrivial fraction of latent coordinates are exactly zero, as a learnable relaxation that may simplify the dynamics.

We introduce LpWorldModel (\textbf{LpWM}), a JEPA world model regularized with RDMReg \citep{kuang2026rectifiedlpjepajointembeddingpredictive} to learn non-negative, sparse codes. RDMReg matches rectified feature distributions to Rectified Generalized Gaussian targets, whose underlying Generalized Gaussian distributions are maximum-entropy under expected $\ell_p$-norm constraints with $p\leq1$ yields sparse targets.

Empirically, sparsity lowers the predictor complexity required for successful planning: predictors that fail over dense codes succeed over sparse ones. The resulting representations are also mode-factored: their support encodes discrete dynamics regimes, while feature magnitudes encode continuous within-regime state. Together, these results suggest that sparse latent geometry can make dynamics both easier to model and more interpretable.
\begin{figure*}[t!]
    \centering
    \begin{subfigure}[t]{0.512\linewidth}
        \centering
        \includegraphics[width=\linewidth]{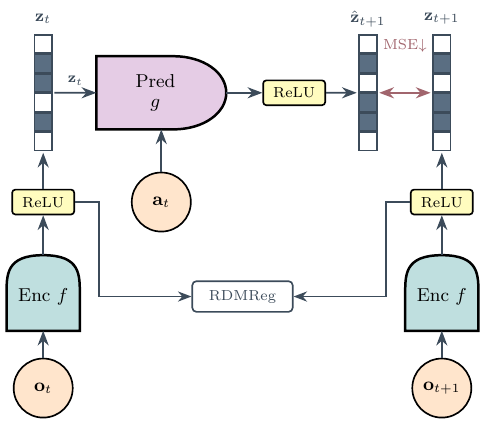}
        \captionsetup{justification=centering}
        \caption{LpWM Schematic Diagram}
        \label{fig:my_figure}
    \end{subfigure}
    \hfill
    \begin{subfigure}[t]{0.448\linewidth}
        \centering
        \includegraphics[width=\linewidth]{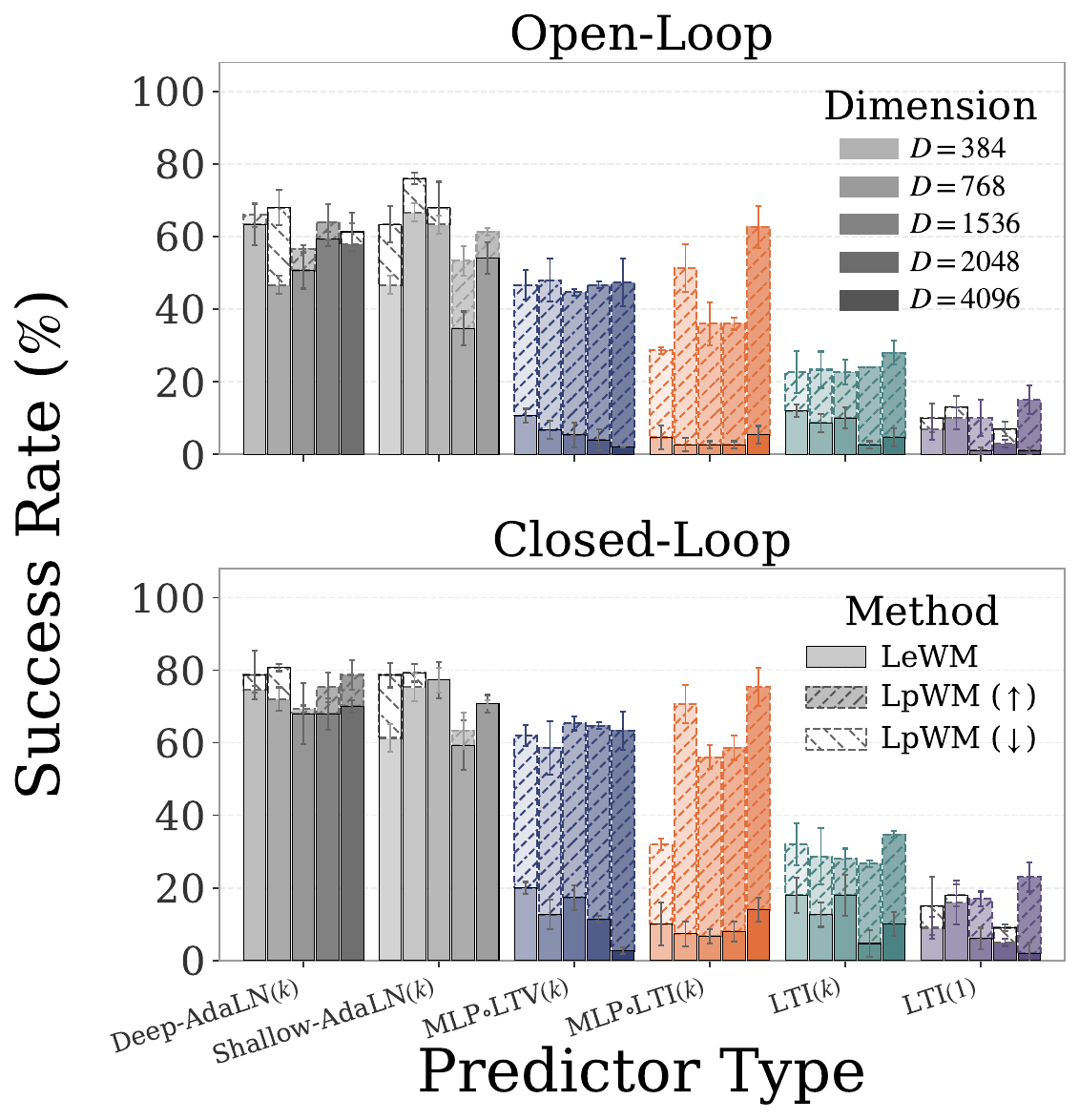}
        \captionsetup{justification=centering}
        \caption{PushT Planning Performance.}
        \label{fig:sparsitylinearityfigsub22}
    \end{subfigure}
    \caption{\textbf{LpWorldModel: a JEPA model with sparse representations.} (a) An encoder $f$ maps observation $\mathbf{o}_t$ to non-negative, sparse latents $\mathbf{z}_{t}$ via a terminal (Rep)ReLU function. The predictor $g$ takes inputs $\mathbf{z}_{t}$ and actions $\mathbf{a}_{t}$ to predict $\hat{\mathbf{z}}_{t+1}$, compared against the encoded target $\mathbf{z}_{t+1}=f(\mathbf{o}_{t+1})$. RDMReg \citep{kuang2026rectifiedlpjepajointembeddingpredictive} matches per-timestep latent marginals to a Rectified Generalized Gaussian, preventing collapse and inducing sparsity. (b). Arrows indicate the success-rate difference between LpWM and LeWM: $\uparrow$ denotes higher performance for LpWM, while $\downarrow$ denotes lower performance. LpWM outperforms LeWM with more linear predictor variants (MLP$\circ$LTV$(k)$, MLP$\circ$LTI$(k)$, and LTI$(k)$. See full specifications in Table~\ref{tab:predictor-ladder}) by up to $57\%$ in success rate, while performing on par with LeWM using the high-capacity DiT predictors. This suggests that sparse representations induce a latent space with simpler dynamics to learn compared to their dense counterparts.}
    \label{fig:sparsitylinearityfigaa}
\end{figure*}

\paragraph{Contributions.}
\begin{itemize}
\item \textbf{LpWorldModel}. We introduce LpWM, a JEPA world model that uses RDMReg to learn non-negative sparse latent representations for latent dynamics modeling and planning.
\item \textbf{Sparse representations for latent linearization}. We formalize the intuition that sufficiently high-dimensional one-hot representations can approximate nonlinear controlled dynamics with linear latent transitions, motivating distributed sparse codes as a practical relaxation.
\item \textbf{Lower-complexity latent dynamics}. We show empirically that sparsity reduces the predictor capacity needed for successful planning: while Wall dynamics are already solved by linear predictors, on the harder PushT environment sparse LpWM substantially outperforms dense LeWM with intermediate-capacity predictors, by up to $57\%$ in success rate.
\item \textbf{Mode-factored latent dynamics}. We find that sparse representations organize dynamics into interpretable factors, with support encoding discrete dynamical regimes and feature magnitudes capturing continuous within-regime state.
\end{itemize}

\section{Method}
\label{sec:method}

\subsection{LpWM: JEPA with Rectified Distribution Matching Regularization (RDMReg)}
\label{sec:method-rdm}

Let $\mathbf{z}_{t+1}=f_{\boldsymbol{\theta}}(\mathbf{x}_{t+1})\in\mathbb{R}^{d}$ be the encoder representation of the observation $\mathbf{x}_{t+1}$ and $\hat{\mathbf{z}}_{t+1}=g_{\boldsymbol{\phi}}(\mathbf{z}_{t},\mathbf{a}_{t})$ be the predictor output with action inputs $\mathbf{a}_{t}\in\mathbb{R}^{d_a}$. The action-conditioned JEPA uses the following loss function to jointly train the encoder $f_{\boldsymbol{\theta}}$ and the predictor $g_{\boldsymbol{\phi}}$:
\begin{align}
\min_{\boldsymbol{\theta},\boldsymbol{\phi}}\quad\|\hat{\mathbf{z}}_{t+1}-\mathbf{z}_{t+1}\|_2 + \lambda_{\operatorname{RDMReg}}\cdot\mathcal{R}(\mathbf{z}_{t+1})
\end{align}
where $\lambda_{\operatorname{RDMReg}}$ is a hyperparameter and $\mathcal{R}(\cdot)$ is an anti-collapse regularizer. We focus on the setting where $\mathcal{R}(\cdot)$ is the Rectified Distribution Matching Regularization (RDMReg) \citep{kuang2026rectifiedlpjepajointembeddingpredictive}:
\begin{align}
    \mathcal{R}(\mathbf{z})=\mathbb{E}_{\mathbf{c}\sim\operatorname{Unif}(\mathbb{S}^{d-1}_{\ell_2})}\big[\mathcal{L}(\mathbb{P}_{\mathbf{c}^\top\mathbf{z}}\|\mathbb{P}_{\mathbf{c}^\top\mathbf{y}})\big]
\end{align}
where the random projection vector $\mathbf{c}\in\mathbb{R}^{d}$ is sampled uniformly from the $\ell_2$ unit sphere $\mathbb{S}_{\ell_2}^{d-1}:=\{\mathbf{x}\in\mathbb{R}^d$ $\vert$ $\|\mathbf{x}\|_2=1\}$, the target random vector $\mathbf{y}\sim\prod_{i=1}^{d}\operatorname{ReLU}(\mathcal{GN}_{p}(\mu,\sigma))$ follows the Rectified Generalized Gaussian (RGG) distribution, $\mathbb{P}_{x}$ denotes the probability distribution for the random variable $x$, and $\mathcal{L}(\cdot\|\cdot)$ is a loss function that compute discrepancies between probability distributions. In our case, we follow \citet{kuang2026rectifiedlpjepajointembeddingpredictive} and set $\mathcal{L}(\cdot\|\cdot)$ to be the $2$-Wasserstein distance. By default we choose $\mu=0$, $\sigma=\sqrt{1/2}$, and $p=1$ and hence our target distribution reduces to the Rectified Laplace distribution \citep{kuang2026rectifiedlpjepajointembeddingpredictive}. Note that when $\mu=0$, $\sigma=1$, $p=2$, and in the absence of $\operatorname{ReLU}$, we recover dense isotropic Gaussian used in LeWM \citep{maes2026leworldmodelstableendtoendjointembedding}.

\subsection{Architecture and Rectification}
\label{sec:method-arch}
Following \citet{maes2026leworldmodelstableendtoendjointembedding}, our encoder $f_{\boldsymbol\theta}$ is a ViT \citep{dosovitskiy2021imageworth16x16words} with the output CLS token passed through a three-layer MLP projector to produce the embedding $\mathbf{z}\in\mathbb{R}^{d}$. The
predictor $g_{\boldsymbol\phi}$ is a Transformer with AdaLN-zero action conditioning \citep{peebles2023scalablediffusionmodelstransformers} and a final three-layer MLP projector that predicts the next embedding $\hat{\mathbf{z}}_{t+1}$. 

To induce exact zeros, both encoder and predictor MLP projector outputs pass through a reparameterized ReLU \citep{wang2024nonnegativecontrastivelearning,kuang2026rectifiedlpjepajointembeddingpredictive}: $\operatorname{RepReLU}(x)=\operatorname{sg}(\operatorname{ReLU}(x))+\operatorname{GeLU}(x)-\operatorname{sg}(\operatorname{GeLU}(x))$, where $\operatorname{sg}(\cdot)$ denotes the stop-gradient operator. In the forward pass, $\operatorname{RepReLU}(x)=\operatorname{ReLU}(x)$, yielding exact zeros, while in the backward pass gradients flow through $\operatorname{GeLU}(x)$ to mitigate the dying-ReLU pathology \citep{hendrycks2023gaussianerrorlinearunits}. Importantly, LpWM does not rely on stop-gradient for collapse prevention: standard $\operatorname{ReLU}$ with RDMReg trains without collapse and yields exactly sparse representations \citep{kuang2026rectifiedlpjepajointembeddingpredictive}. We therefore use $\operatorname{RepReLU}$ as an optimization safeguard against dead neurons rather than as an essential component of the method. See Figure~\ref{fig:my_figure} for LpWM architecture.

\subsection{Planning with the Learned World Model}
\label{sec:method-planning}
At test time the world model serves as the dynamics model for model-predictive control (MPC): we encode the current and goal observation images to $\mathbf{z}_0$ and $\mathbf{z}_g$, and search for an action sequence whose predicted latent rollout minimizes the goal cost $\mathcal{C}=\|\hat{\mathbf{z}}_T-\mathbf{z}_g\|_2$ using the Cross-Entropy Method (CEM) \citep{rubinstein1999cross,zhou2025dinowmworldmodelspretrained}. See additional experimental details in Appendix~\ref{app:experimental-details}.

\begin{figure*}[t!]
    \centering
    \begin{subfigure}[t]{0.32\linewidth}
        \centering
        \includegraphics[width=\linewidth]{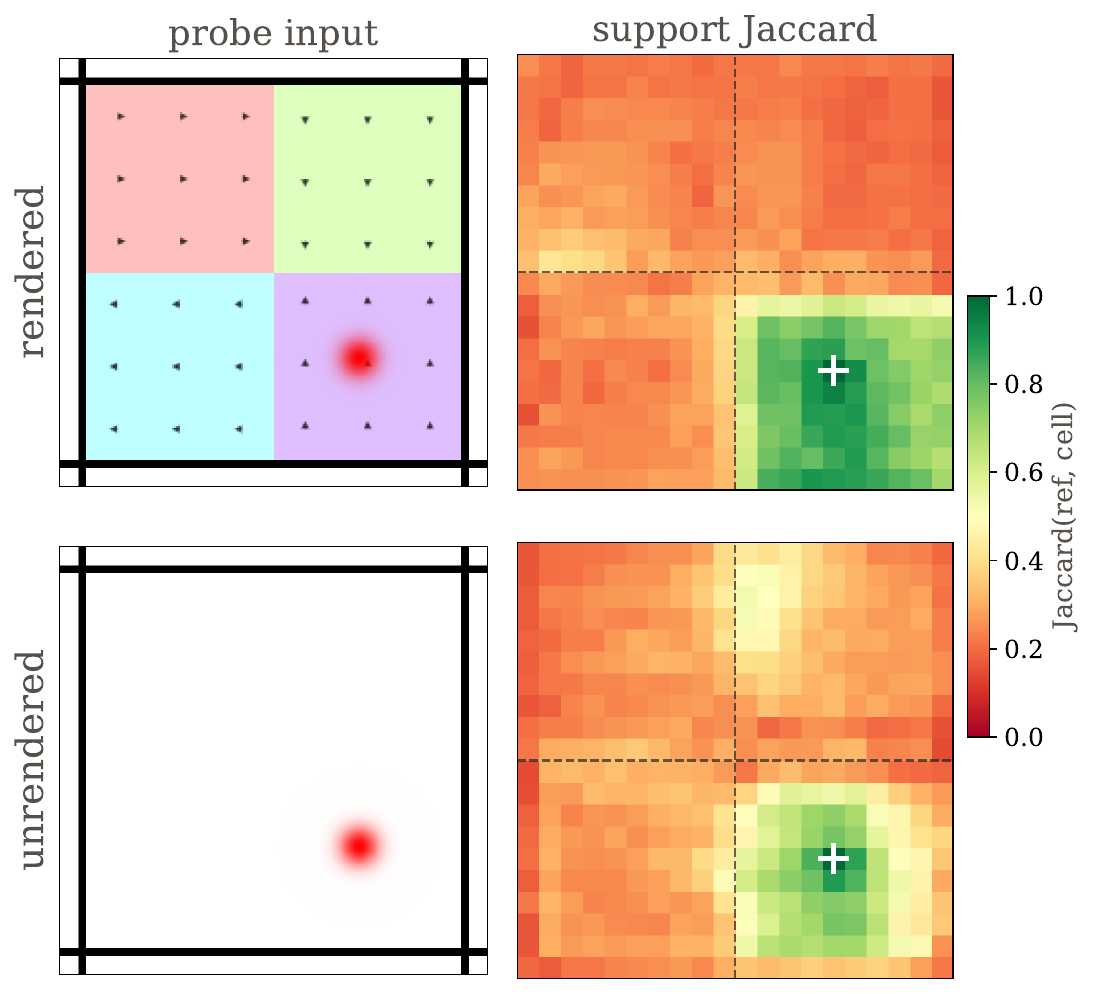}
        \captionsetup{justification=centering}
        \caption{Support Jaccard Heatmap.}
        \label{fig:aaa1}
    \end{subfigure}
    \hfill
    \begin{subfigure}[t]{0.32\linewidth}
        \centering
        \includegraphics[width=\linewidth]{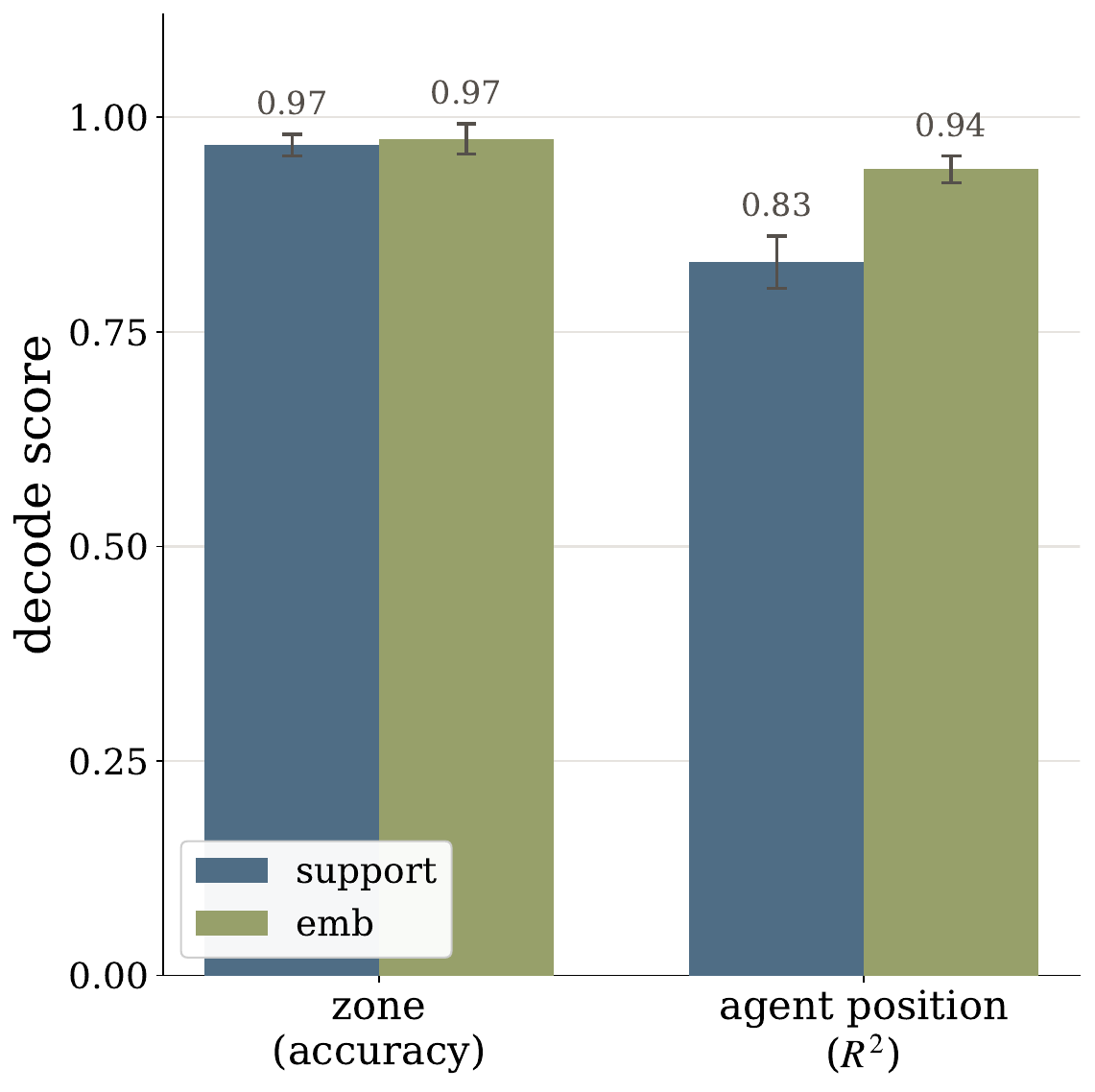}
        \captionsetup{justification=centering}
        \caption{State Decoding.}
        \label{fig:aaa2}
    \end{subfigure}
    \hfill
    \begin{subfigure}[t]{0.32\linewidth}
        \centering
        \includegraphics[width=\linewidth]{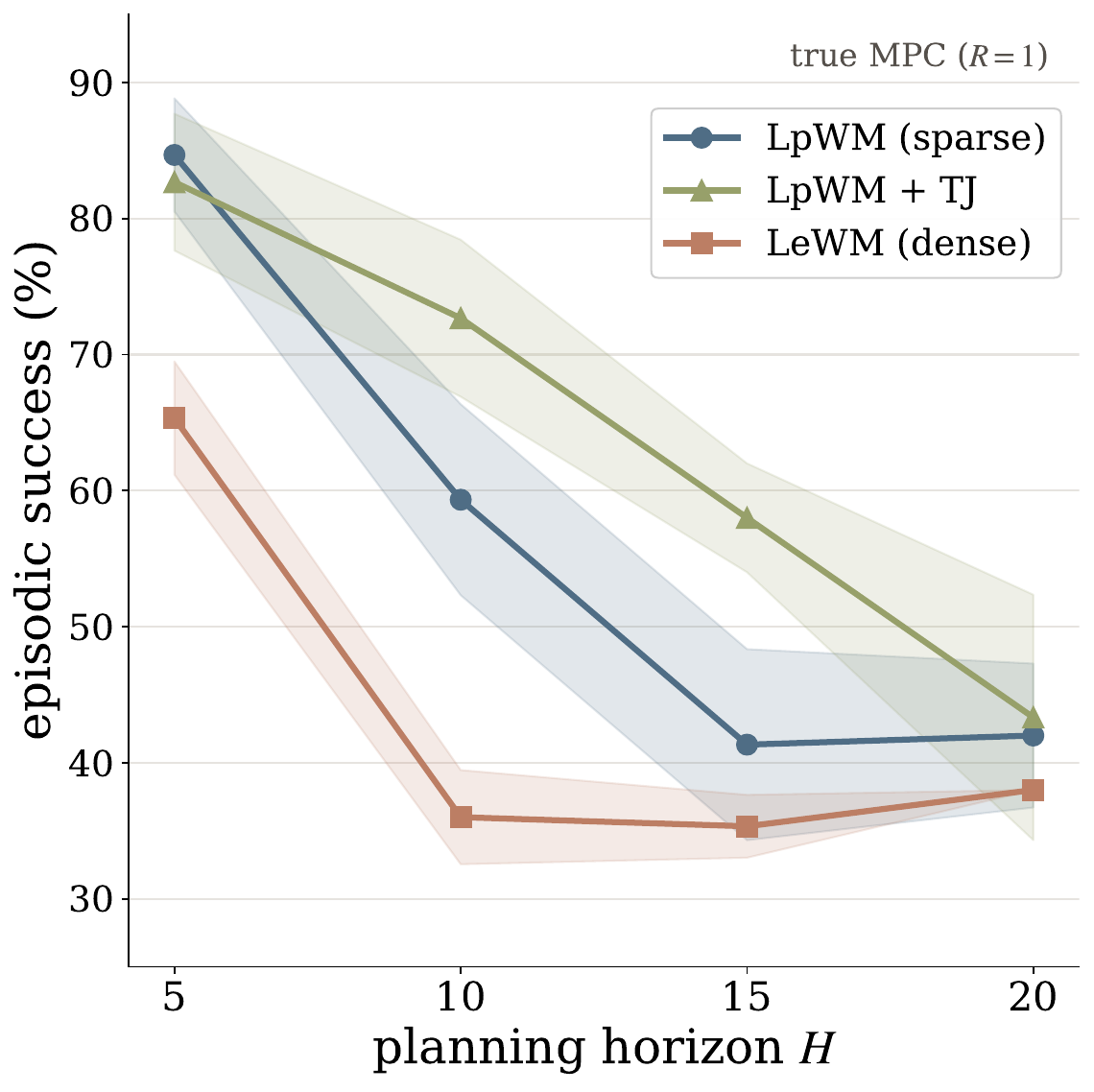}
        \captionsetup{justification=centering}
        \caption{Planning Evaluation.}
        \label{fig:aaa3}
    \end{subfigure}
    \caption{\textbf{The sparse support recovers the piecewise structure of the dynamics.} \textbf{(a)} Support Jaccard heatmap over a $20\times20$ grid: the Jaccard index of each cell's support with a fixed reference cell (white $+$) is high within the reference zone and drops across zone boundaries (dashed), for a model trained with rendered zones (top) and one with no visual cues (bottom). \textbf{(b)} The binary support decodes the ground-truth zone identity as accurately as the full embedding, while agent position is decoded better from the magnitudes than from the support. \textbf{(c)} Planning success vs.\ horizon (closed-loop, 3-seed mean$\pm$std). Goal locations of the dot agent are sampled randomly. Both sparse variants beat dense LeWM at every horizon.}
    \label{fig:aaa4}
\end{figure*}

\section{Does Sparsity Simplify Latent Dynamics?}\label{sec:doessparsitysimplify}

In the following sections, we ask whether a sparse code makes the action-conditioned latent dynamics simpler by requiring a less expressive predictor than a dense code. We first show that a sufficiently high-dimensional sparse code renders the latent dynamics exactly linear under appropriate assumptions (Proposition~\ref{prop:linearapproxofnonlinear}). We then relax this to a learnable distributed sparse code and test, along a predictor-complexity ladder (Table~\ref{tab:predictor-ladder}), whether a simpler predictor plans better under sparsity.

\subsection{Linearization of Nonlinear Dynamics with Sparse Representations}

In Proposition~\ref{prop:linearapproxofnonlinear}, we show that compact controlled systems with Lipschitz dynamics admit finite one-hot sparse encodings with exactly linear latent dynamics. This is closely related to classical symbolic abstractions based on state-space quantization \citep{pola2008approximately}, which we reinterpret as one-hot sparse latent codes.

\begin{proposition}[Linear Approximation of Nonlinear Dynamics with Sparse Representations]\label{prop:linearapproxofnonlinear}
Let $\mathcal{X}\subset\mathbb{R}^d$ be a compact state space and let
$\mathcal{A}\subset\mathbb{R}^p$ be a compact continuous action space.
Consider a deterministic discrete-time controlled dynamical system $\mathbf{x}_{t+1} = f(\mathbf{x}_t,\mathbf{a}_t)$ where $f:\mathcal{X}\times\mathcal{A}\rightarrow\mathcal{X}.$ Assume that $f$ is uniformly Lipschitz in the state variable, i.e. there exists a
finite constant $L\ge 0$ such that $\left\|
f(\mathbf{x},\mathbf{a})
-
f(\mathbf{y},\mathbf{a})
\right\|
\le
L\|\mathbf{x}-\mathbf{y}\|
$ for every $\mathbf{x},\mathbf{y}\in\mathcal{X}$ and
$\mathbf{a}\in\mathcal{A}$. Then, for every $\varepsilon>0$, there exist 1) a finite latent dimension $N$; 2) a one-hot sparse encoder $E_\varepsilon:\mathcal{X}\rightarrow\mathbb{R}^N$; 3) a decoder $D_\varepsilon:\{\mathbf{e}_1,\ldots,\mathbf{e}_N\}\rightarrow\mathcal{X}$; and 4) an action-conditioned family of matrices $P_\varepsilon : \mathcal{A}\rightarrow\{0,1\}^{N\times N},$ $\mathbf{a} \mapsto \mathbf{P}_\varepsilon(\mathbf{a}),$ such that the latent dynamics $\mathbf{z}_{t+1}=\mathbf{P}_\varepsilon(\mathbf{a}_t)\mathbf{z}_t$ are exactly linear in $\mathbf{z}_t$, and the decoded one-step prediction satisfies
\begin{align}
\sup_{\mathbf{x}_t\in\mathcal{X},\,\mathbf{a}_t\in\mathcal{A}}
\left\|
f(\mathbf{x}_t,\mathbf{a}_t)
-
D_\varepsilon\!\left(
\mathbf{P}_\varepsilon(\mathbf{a}_t)
E_\varepsilon(\mathbf{x}_t)
\right)
\right\|
\le
(L+1)\varepsilon.
\end{align}
\end{proposition}
\begin{proof}
    See Appendix \ref{proof:linapproxofnonlinearwithsparse}
\end{proof}
In Corollary~\ref{coro:sparsityinhighd}, we consider one example of a compact domain and show that the rollout error with respect to the ground-truth dynamics vanishes as the encoding dimension grows.

\begin{corollary}[Sparsity in High-Dimensions]\label{coro:sparsityinhighd}
Let the state space be $\mathcal{X}=[0,1]^d$ and consider dividing each coordinate into $n$ equal intervals, producing $N=n^d$ grid cells. We represent each cell by its center with the covering radius $\varepsilon_N=\frac{\sqrt{d}}{2}N^{-1/d}$. Hence the one-step prediction error satisfies
\begin{align}
\sup_{\mathbf{x},\mathbf{a}}
\left\|
f(\mathbf{x},\mathbf{a})
-
D_{\varepsilon_N}\!\left(
\mathbf{P}_{\varepsilon_N}(\mathbf{a})E_{\varepsilon_N}(\mathbf{x})
\right)
\right\|
\le
\frac{\sqrt{d}}{2}(L+1)N^{-1/d}=
O(N^{-1/d}).
\end{align}
For any fixed rollout horizon $H$, $\|\mathbf{x}_H-\hat{\mathbf{x}}_H\|$ $\le\frac{\sqrt{d}}{2}N^{-1/d}\sum_{k=0}^{H}L^k$ and $\lim_{N\rightarrow\infty}$ $\|\mathbf{x}_H-\hat{\mathbf{x}}_H\|=0$.
\end{corollary}
\begin{proof}
    See Appendix \ref{proof:sparsityinhighdimproof}.
\end{proof}
We note that the approximation error of one-step prediction under an exact one-hot sparse encoding decreases as $O(N^{-1/d})$, and therefore suffers from the curse of dimensionality. Nonetheless, Corollary~\ref{coro:sparsityinhighd} illustrates a useful trade-off: higher-dimensional sparse representations can simplify the latent transition map, motivating distributed sparse codes as a practical relaxation that may reduce the complexity of the learned dynamics model. In the following experiments, we test whether this simplification enables successful planning with less expressive predictors. In the special case where the learned latent dynamics reduce further to a standard linear controlled system, classical control methods such as LQR become directly applicable under standard assumptions \citep{kalman1960contributions}.

\subsection{Experiment Designs}

In practice, exact high-dimensional one-hot encodings are infeasible to optimize, and we need exponentially large feature dimensions to linearize the latent dynamics. LpWM instead learns distributed sparse codes whose sparsity is controlled by the target-distribution parameter $\mu$. Although these codes maybe not linearize the dynamics exactly, we hypothesize that they make the latent dynamics simpler to model than dense representations, so that lower-capacity predictors can suffice.

To test this, we consider a list of predictor options in Table~\ref{tab:predictor-ladder}, ranging from nonlinear DiT predictors with AdaLN action conditioning \citep{peebles2023scalablediffusionmodelstransformers} to linear time-invariant (LTI) state-space models. We train sparse LpWM and dense LeWM with the same encoder architecture described in Section~\ref{sec:method-arch}, but replace the standard DiT predictor with candidates in Table~\ref{tab:predictor-ladder} across latent dimensions $D\in\{384,768,1536,2048,4096\}$ with their corresponding parameter counts in Table~\ref{tab:predictor-params}.

Additionally, we test a variant of LeWM in which SIGReg \citep{balestriero2025lejepaprovablescalableselfsupervised} is replaced with VICReg \citep{bardes2022vicregvarianceinvariancecovarianceregularizationselfsupervised}, yielding another JEPA model with dense representations. For simplicity, we use VICReg to denote this LeWM variant in Figure~\ref{fig:sparsitylinearityfigsub2vicreg}.

Since our experiment design involves comparing models under width scaling, we use maximum update parameterization ($\mu$P) to stabilize training across widths, with details in Appendix~\ref{app:mupforjepa}. Experiments are conducted on Wall and PushT environments (Appendix~\ref{app:envexplained}).

\begin{table}[t]
  \centering\small
  \caption{\textbf{Predictors with Different Levels of Complexity.} $\hat{\mathbf{z}}_{t+1}$ is the predicted latent, $\mathbf{z}_{t-i}$ past latents, and $\mathbf{a}_t$ the action. AdaLN predictors are Transformer-based models with Adaptive Layer Normalization for action conditioning, i.e. DiT architecture \citep{peebles2023scalablediffusionmodelstransformers}. The Deep-AdaLN$(k)$ predictor is the one used in LeWM \citep{maes2026leworldmodelstableendtoendjointembedding}. LTI denotes a linear time-invariant predictor with fixed operators $\mathbf{A}_i,\mathbf{B}$, while LTV denotes a linear time-varying predictor with state-dependent operators $\mathbf{A}_i(\mathbf{z}_t),\mathbf{B}(\mathbf{z}_t)$. We use MLP$\circ$ to denote one extra nonlinear transform. The history $k$ denotes context length, with $1$ the single-frame case. The output link is $\sigma=\mathrm{identity}$ for dense LeWM and $\sigma=\mathrm{RepReLU}$ for sparse LpWM. See Appendix~\ref{app:preddesignsec} for details.}
  \label{tab:predictor-ladder}
  \begin{tabular}{lccl}
    \toprule
    Predictor & Layers & History ($k$) & Functional form \\
    \midrule
    \textbf{Deep-AdaLN}$(k)$    & $6$ & $3$ & $\hat{\mathbf{z}}_{t+1}=\sigma\!\left(\mathrm{AdaLN}^{(6)}(\mathbf{z}_{t-k+1:t},\,\mathbf{a}_t)\right)$ \\
    \textbf{Shallow-AdaLN}$(k)$ & $1$ & $3$ & $\hat{\mathbf{z}}_{t+1}=\sigma\!\left(\mathrm{AdaLN}^{(1)}(\mathbf{z}_{t-k+1:t},\,\mathbf{a}_t)\right)$ \\
    \textbf{MLP$\circ$LTV}$(k)$ & $1$ & $3$ & $\hat{\mathbf{z}}_{t+1}=\sigma\!\left(\mathbf{W}\,\mathrm{ReLU}(\textstyle\sum_{i=0}^{k-1}\mathbf{A}_i(\mathbf{z}_t)\,\mathbf{z}_{t-i}+\mathbf{B}(\mathbf{z}_t)\,\mathbf{a}_t)\right)$ \\
    \textbf{MLP$\circ$LTI}$(k)$ & $1$ & $3$ & $\hat{\mathbf{z}}_{t+1}=\sigma\!\left(\mathbf{W}\,\mathrm{ReLU}(\textstyle\sum_{i=0}^{k-1}\mathbf{A}_i\,\mathbf{z}_{t-i}+\mathbf{B}\,\mathbf{a}_t)\right)$ \\
    \textbf{LTI}$(k)$           & $1$ & $3$ & $\hat{\mathbf{z}}_{t+1}=\sigma\!\left(\textstyle\sum_{i=0}^{k-1}\mathbf{A}_i\,\mathbf{z}_{t-i}+\mathbf{B}\,\mathbf{a}_t\right)$ \\
    \textbf{LTI}$(1)$           & $1$ & $1$ & $\hat{\mathbf{z}}_{t+1}=\sigma\!\left(\mathbf{A}\,\mathbf{z}_t+\mathbf{B}\,\mathbf{a}_t\right)$ \\
    \bottomrule
  \end{tabular}
\end{table}

\begin{table}[t]
  \centering 
  \caption{\textbf{Predictor parameter counts across latent dimensions $D$.}
  Counts are for the predictor alone for each
  $D$ used in our experiments. LTI, MLP$\circ$LTI, and MLP$\circ$LTV are attention-free and far smaller than the AdaLN Transformers.}
  \label{tab:predictor-params}
  \begin{tabular}{lrrrrr}
    \toprule
    Predictor & $D{=}384$ & $D{=}768$ & $D{=}1536$ & $D{=}2048$ & $D{=}4096$ \\
    \midrule
    \textbf{Deep-AdaLN}$(k)$    & $25.8$M & $62.2$M & $166.9$M & $260.2$M & $822.4$M \\
    \textbf{Shallow-AdaLN}$(k)$ & $5.6$M  & $13.0$M & $33.1$M  & $50.4$M  & $151.1$M \\
    \textbf{MLP$\circ$LTV}$(k)$ & $0.81$M & $3.10$M & $12.1$M  & $21.4$M  & $84.7$M \\
    \textbf{MLP$\circ$LTI}$(k)$ & $0.74$M & $2.95$M & $11.8$M  & $21.0$M  & $83.9$M \\
    \textbf{LTI}$(k)$           & $0.59$M & $2.36$M & $9.44$M  & $16.8$M  & $67.1$M \\
    \textbf{LTI}$(1)$           & $0.30$M & $1.18$M & $4.72$M  & $8.39$M  & $33.6$M \\    \bottomrule
  \end{tabular}
\end{table}

\subsection{Results}

\begin{figure*}[t!]
    \centering
    \begin{subfigure}[t]{0.5113\linewidth}
        \centering
        \includegraphics[width=\linewidth]{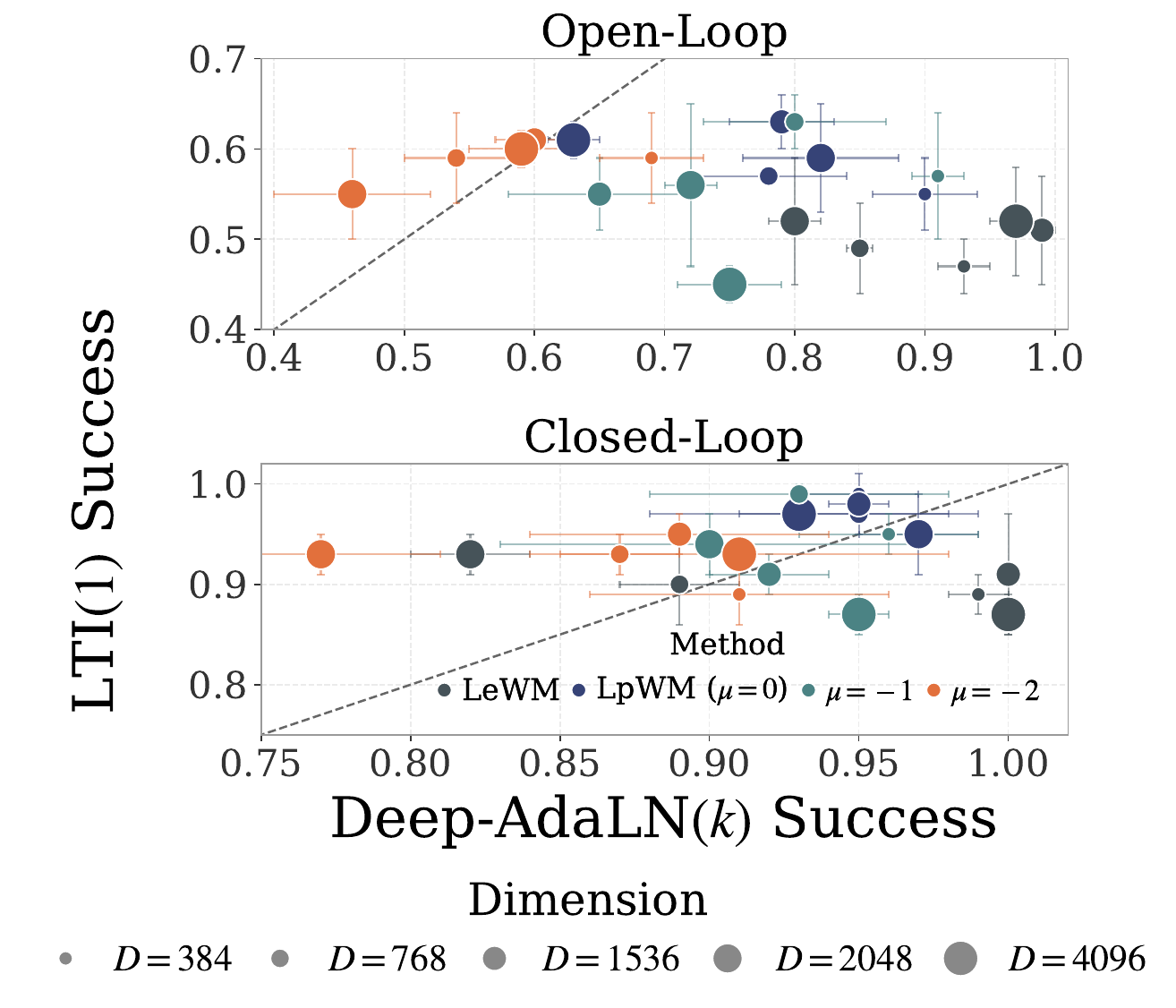}
        \captionsetup{justification=centering}
        \caption{Wall Environment. LeWM vs. LpWM}
        \label{fig:sparsitylinearityfigsub1}
    \end{subfigure}
    \hfill
    \begin{subfigure}[t]{0.4686\linewidth}
        \centering
        \includegraphics[width=\linewidth]{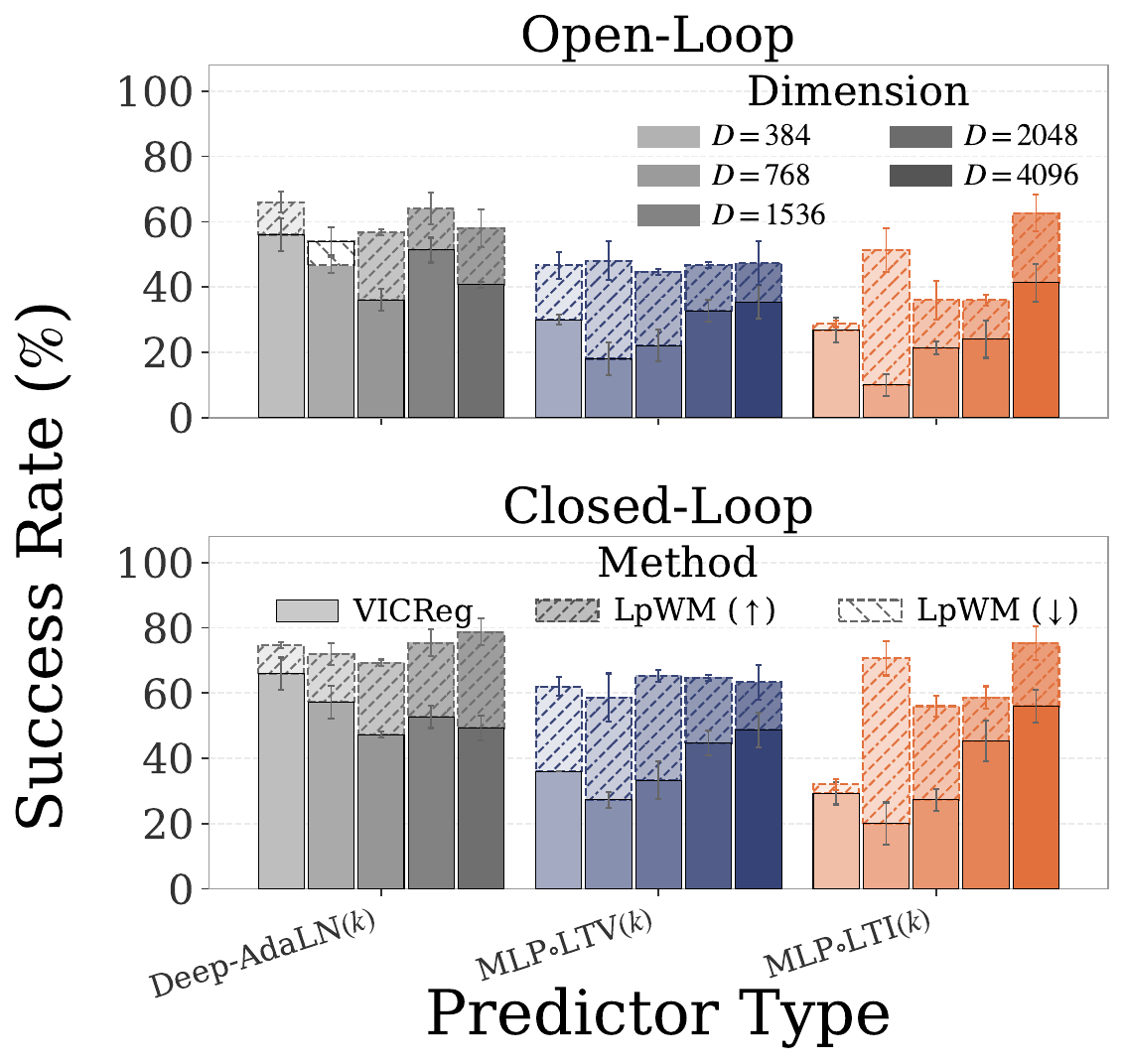}
        \captionsetup{justification=centering}
        \caption{PushT Environment. VICReg vs. LpWM}
        \label{fig:sparsitylinearityfigsub2vicreg}
    \end{subfigure}
    \caption{\textbf{Linear predictor is sufficient for the wall environment regardless of dense or sparse embeddings, and LpWM outperforms VICReg-based JEPA across predictor settings in PushT.} \textbf{(a)} We plot the planning success rate of the linear LTI$(1)$ predictor against the nonlinear Deep-AdaLN$(k)$, where each dot denotes one setting of method and feature dimension. For simplicity, we use $\mu=*$ to denote LpWM with the target distribution parameter $\mu$. We show the open-loop evaluation on top and closed-loop results at bottom. Both are ran using CEM planner. We observe that the closed-loop success rate for LTI$(1)$ predictor using SIGReg is already approaching $100\%$, indicating that the wall environment can already be described well by a linear dynamics model in the latent space. (b) We have already shown LeWM vs. LpWM in Figure~\ref{fig:sparsitylinearityfigsub22}. Here we compare LpWM with a variant of LeWM where we replaced SIGReg with VICReg in the PushT environment across predictor settings. As oppposed to SIGReg, we observe consistent improvement from LpWM over VICReg even for the Deep-AdaLN$(k)$ predictor, and the gap between sparse LpWM and dense methods (i.e. VICReg here) continues to hold for MLP$\circ$LTV$(k)$ and MLP$\circ$LTI$(k)$ predictors. See full hyperparameter sweeps in Appendix~\ref{subsec:vicregsweepinpusht}.}
    \label{fig:sparsitylinearityfig}
\end{figure*}

In Figure~\ref{fig:sparsitylinearityfigsub1}, we observe that for the Wall environment, even the simplest linear LTI$(1)$ predictor reaches nearly $100\%$ closed-loop CEM success rate for both LeWM and LpWM, leaving little room to distinguish sparse and dense representations.

We therefore turn to the harder PushT environment with results shown in Figure~\ref{fig:sparsitylinearityfigsub22}. At the lowest predictor capacity, LTI$(1)$ fails for both LeWM and LpWM. Thus PushT differs from Wall in that the underlying dynamics cannot be modelled well by a linear predictor in practice.

At the highest capacity, Deep-AdaLN$(k)$ and Shallow-AdaLN$(k)$ perform similarly for both representations, indicating that the predictor complexity saturates and we don't observe significant advantages of sparse LpWM against dense LeWM.

The difference emerges at intermediate capacity: on PushT, sparse LpWM improves planning success over dense LeWM by $24\%$--$57\%$ with MLP$\circ$LTI$(k)$, $36\%$--$45\%$ with MLP$\circ$LTV$(k)$, and $11\%$--$23\%$ with LTI$(k)$. These performance gaps are verified across extensive hyperparameter sweeps for both sparse LpWM and dense LeWM in Appendix~\ref{subsec:lpwmvslewmsweepinpusht}. We also report the $\ell_0$ norms of the LpWM embeddings in Table~\ref{tab:l0-sparsity}.

Additionally, Figure~\ref{fig:sparsitylinearityfigsub2vicreg} shows that LpWM outperforms VICReg across Deep-AdaLN$(k)$, MLP$\circ$LTV$(k)$, and MLP$\circ$LTI$(k)$ predictors. This indicates that the benefit of sparsity extends beyond comparison with the specific dense isotropic Gaussian targets used in LeWM to more general dense representations. Moreover, LpWM outperforms VICReg even with the high-capacity Deep-AdaLN$(k)$ predictor. We hypothesize that this is because second-order moment constraints in VICReg alone can be insufficient as anti-collapse targets because they leave the target distribution underspecified.

Overall, these results show that the benefit of sparsity depends on predictor capacity relative to dynamics complexity. In the simple Wall environment, even linear predictors are sufficient, leaving little room for sparsity to improve performance. In the harder PushT environment, sparsity provides substantial gains at intermediate capacities, but these gains diminish with sufficiently expressive predictors. As the underlying dynamics become more complex, the advantages of sparsity may extend to increasingly expressive predictors, potentially making sparse representations a preferable default over dense representations for modeling latent dynamics. We leave a systematic investigation of this hypothesis to future work.

\section{Does Sparsity Lead to Interpretable Latent Dynamics?}\label{sec:sparseforinterpretability}

Section~\ref{sec:doessparsitysimplify} shows that sparse codes can simplify latent prediction. We now ask what structure emerges in these representations. We call a representation mode-factored when its binary support primarily identifies the discrete dynamics regime, while feature magnitudes capture finer within-regime state. We first test whether this structure emerges naturally in Piecewise, where the dynamics mode is known by construction, and then examine whether it extends to OGBench-Cube \citep{park2025ogbenchbenchmarkingofflinegoalconditioned}, where contact provides a natural temporal regime change.

\begin{figure*}[t!]
    \centering
    \begin{subfigure}[t]{0.32\linewidth}
        \centering
        \includegraphics[width=\linewidth]{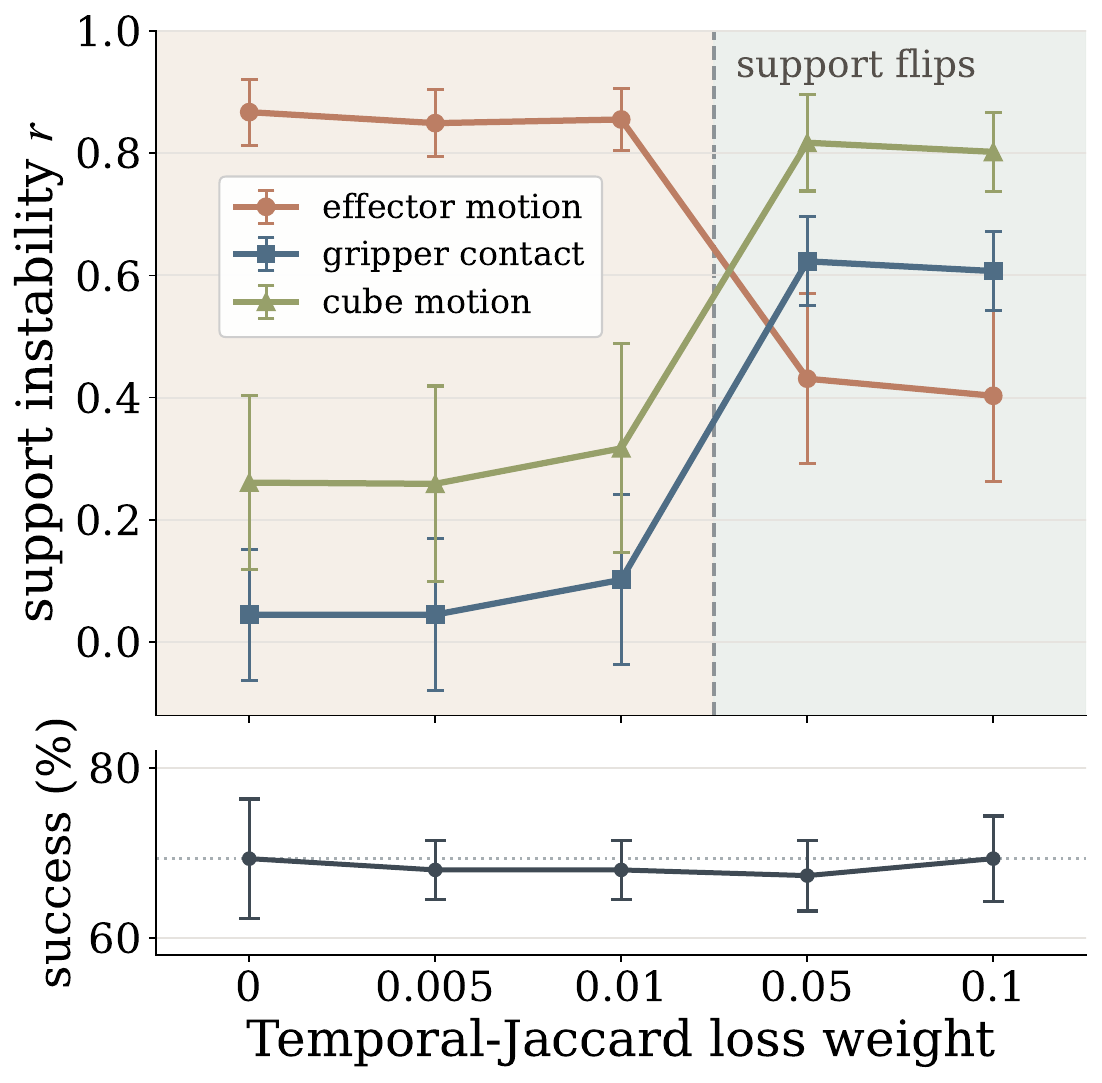}
        \captionsetup{justification=centering}
        \caption{Temporal Jaccard Regularization}
        \label{fig:fourquadrantssubfig1}
    \end{subfigure}
    \hfill
    \begin{subfigure}[t]{0.32\linewidth}
        \centering
        \includegraphics[width=\linewidth]{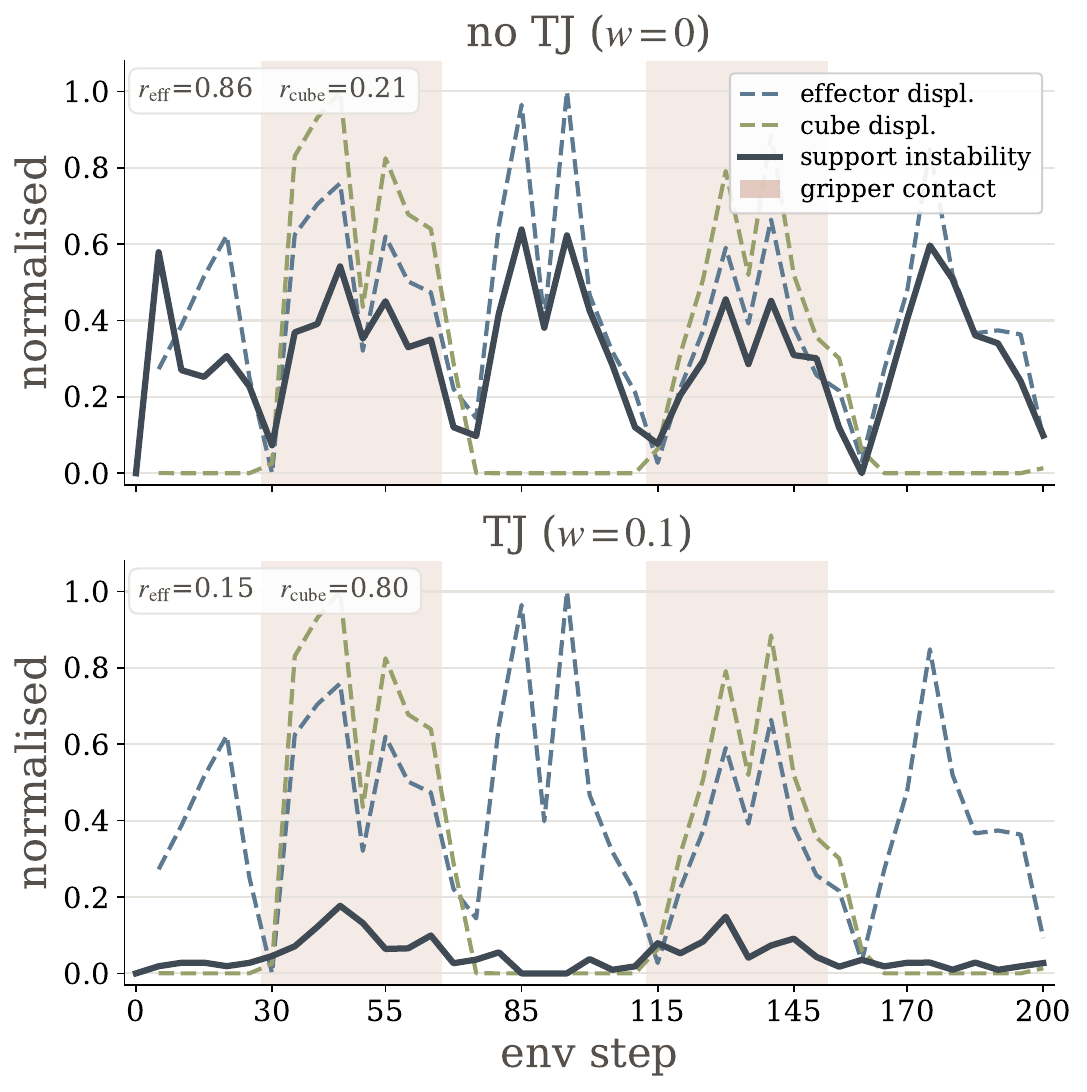}
        \captionsetup{justification=centering}
        \caption{Example Episode.}
        \label{fig:controllablesparsitysubfig2}
    \end{subfigure}
    \hfill
    \begin{subfigure}[t]{0.32\linewidth}
        \centering
        \includegraphics[width=\linewidth]{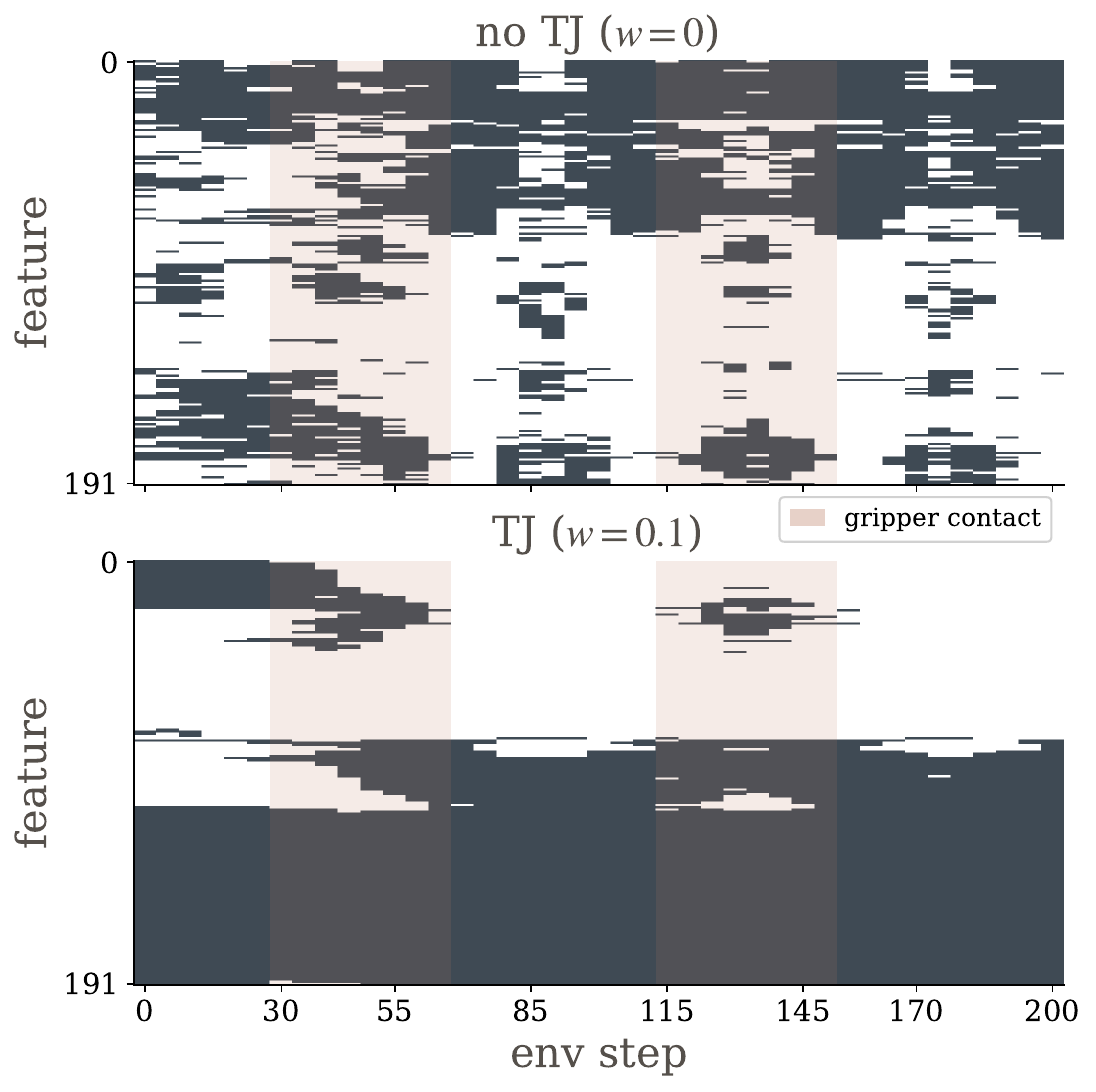}
        \captionsetup{justification=centering}
        \caption{Support Raster Plot.}
        \label{fig:paretofrontieraccandsparsesubfig3}
    \end{subfigure}
    \caption{\textbf{Temporal-Jaccard (TJ) makes the cube support track contact, not effector motion.} \textbf{(a)} As TJ strength increases, support-instability correlation with effector motion drops ($0.87{\to}0.40$) while cube-motion ($0.26{\to}0.80$) and gripper-contact ($0.05{\to}0.61$) rise---at no change in planning success. \textbf{(b)} One example episode: without TJ (top) support instability tracks effector displacement ($r_{\mathrm{eff}}{=}0.86$); with TJ (bottom) it tracks cube displacement ($r_{\mathrm{cube}}{=}0.80$), spiking inside the shaded contact windows. \textbf{(c)} Support raster: without TJ the support rapidly transitions across frame; with TJ it is smoothly varying in time, reorganizing almost during the shaded contact window.}
    \label{fig:TBD1fig}
\end{figure*}

\subsection{Discovering Piecewise-Affine Dynamics}\label{sec:piecewiseanalysisresult}

Piecewise is a 2D navigation environment partitioned into zones with different force fields, yielding piecewise-affine dynamics (Appendix~\ref{app:piecewiseenvintro}). Empirically, we observe that success rate reaches $84.7\%$ for LpWM versus $65.3\%$ for LeWM with random goals, while performance is saturated for goals sampled from the set of evaluation episodes for both methods (Table~\ref{tab:pw-sweep}).

To test whether the sparse support recovers the ground-truth piecewise structure, we discretize the Piecewise environment into a $20\times20$ grid, yielding $400$ possible agent locations. For each location, we compute the corresponding embedding and binarize it into a support vector in $\{0,1\}^D$. For two binary support vectors $\mathbf{x},\mathbf{y}\in\{0,1\}^D$, we measure their overlap using the Jaccard index
\begin{align}
J(\mathbf{x},\mathbf{y})=\frac{\sum_{i=1}^{D}\mathbbm{1}_{\mathbf{x}_i=1\land\mathbf{y}_i=1}}{\sum_{i=1}^{D}\mathbbm{1}_{\mathbf{x}_i=1\lor\mathbf{y}_i=1}},
\end{align}
where $\land,\lor$ are logical ``and'', ``or'' symbols. We defer further details and its differentiable relaxation given in Appendix~\ref{app:jaccardlossessec}. We then select a reference cell, marked by $+$ in Figure~\ref{fig:aaa1}, and compute its Jaccard similarity with every other cell. The resulting similarity is high within the reference zone and drops sharply across zone boundaries (dashed), a pattern that persists even when the zones contain no visual cues (bottom). This indicates that LpWM recovers the underlying dynamics regime from action-conditioned prediction rather than appearance alone, organizing the mode structure into which features are active.

Linear probes further confirm a factorized code (Figure~\ref{fig:aaa2}): the binary support decodes the ground-truth zone at near-perfect accuracy for sparse LpWM, as accurately as the full embedding, while the continuous agent position is decoded better from the magnitudes. Therefore, the support carries the discrete zone mode, and the magnitudes carry the within-zone agent's state.

Finally, this structure also accompanies strong long-horizon planning (Figure~\ref{fig:aaa3}): sparse LpWM outperforms dense LeWM at every horizon $H$, with the gap widening as $H$ increases.

\subsection{From Navigation to Contact-Rich Dynamics}\label{sec:cubeanalysisresult}

Does the sparse code also separate distinct phases of the dynamics in contact-rich manipulation? On OGBench-Cube, contact provides a natural regime change between free arm motion and arm--cube interaction. As a temporal analogue of the spatial partition in Piecewise, we measure the support instability $1-J(\mathbf{z}_t,\mathbf{z}_{t+1})$ between consecutive frames and compute its Pearson correlation with effector motion, cube motion, and gripper contact.

Unlike Piecewise, sparsity alone does not determine which temporal factor is encoded by the support. RDMReg constrains only the per-frame marginal, and without an explicit temporal prior the support tends to follow the dominant fast-varying signal. Accordingly, without temporal regularization, support instability behaves largely as a motion detector: it correlates strongly with effector motion, with $r\approx0.87$, but only weakly with contact, with $r\approx0.05$ (Figure~\ref{fig:fourquadrantssubfig1}).

To test whether the support can instead track slower physical regime changes, we additionally consider an optional temporal Jaccard (TJ) loss that encourages support stability across adjacent frames. As the TJ weight increases, support instability becomes less aligned with effector motion and increasingly aligned with cube motion and contact (Figure~\ref{fig:controllablesparsitysubfig2}); for example, the correlation with cube motion increases from $r_{\mathrm{cube}}=0.21$ to $0.80$. The support raster in Figure~\ref{fig:paretofrontieraccandsparsesubfig3} shows the same effect visually.

These results highlight a limitation of vanilla LpWM: sparsity alone may not ensure that support transitions correspond to semantically meaningful dynamical events, and without temporal structure the support can instead behave largely as a motion detector. However, they also suggest that, with an appropriate temporal prior, sparse support can become an interpretable tracker of physical regime changes such as object interaction and contact. We leave the design of temporal regularization for more complex dynamics to future work.

\section{Conclusion}
\label{sec:discussion}

We introduced LpWM, an action-conditioned JEPA regularized with RDMReg to produce non-negative, exactly sparse latents. Motivated by a one-hot linearization result, we showed empirically that sparsity simplifies the action-conditioned latent dynamics: at a fixed, intermediate predictor capacity, a shallow predictor plans over sparse codes where it fails over dense ones on PushT. We further found the sparse code is mode-factored---its support encodes the discrete dynamics regime ($94$--$99\%$ decodable on Piecewise) and, with a temporal-Jaccard prior, tracks contact on OGBench-Cube. Hence sparse representations can reduce the predictor complexity for control while leading to more interpretable dynamics. 

\section*{Acknowledgement}

We thank Amir Bar, Raktim Goswami, Ying Wang, Wancong Zhang, Gaoyue Zhou, and Kefei Zhu for helpful discussions. This work was supported in part by AFOSR under grant FA95502310139 and NSF Award 1922658. This work was also supported through AMI Labs and NYU IT High Performance Computing resources, services, and staff expertise.

\bibliographystyle{plainnat}
\bibliography{references}


\appendix

\section{Related Work}\label{sec:relatedwork}

\label{sec:related}
\paragraph{Latent world models for planning.} Joint-Embedding Predictive Architectures (JEPA) predict in
a learned latent space rather than reconstructing pixels \citep{assran2023selfsupervisedlearningimagesjointembedding,bardes2024revisitingfeaturepredictionlearning}. Equipped with action-conditioning, JEPA-style world models can be viewed as latent dynamics model of the environment, and can be used for planning with model predictive control \citep{assran2025vjepa2selfsupervisedvideo,zhou2025dinowmworldmodelspretrained,sobal2025learningrewardfreeofflinedata,maes2026leworldmodelstableendtoendjointembedding}. Hierarchical JEPA has also been shown to outperform standard flat JEPA models with improved planning performances \citep{zhang2026hierarchicalplanninglatentworld}. 

\paragraph{Feature Regularization for Latent World Models.} JEPA-style world models plan in the latent space, and hence proper regularizations of the embedding matters for both anti-collapse and also better planning performances. Preventing collapse without negative examples \citep{chen2020simpleframeworkcontrastivelearning} is the central design problem of non-contrastive SSL \citep{balestriero2023cookbookselfsupervisedlearning}. The stop-gradient and EMA techniques \citep{grill2020bootstraplatentnewapproach} have been used in DINO-WM \citep{zhou2025dinowmworldmodelspretrained}, V-JEPA 2 AC \citep{assran2025vjepa2selfsupervisedvideo} and related work. More recently, information maximization based techniques have also been proposed, which aims at reducing feature redundancies \citep{zbontar2021barlowtwinsselfsupervisedlearning,ermolov2021whiteningselfsupervisedrepresentationlearning,yerxa2023learningefficientcodingnatural} and maximizing feature entropy through distribution-matching regularizations \citep{chakraborty2025improvingpretrainedselfsupervisedembeddings,balestriero2025lejepaprovablescalableselfsupervised,kuang2026radialvcreginformativerepresentationlearning,kuang2026rectifiedlpjepajointembeddingpredictive}.

Beyond collapse preventions, other techniques have been proposed to improve the latent space for planning and control. \citet{wang2026temporalstraighteninglatentplanning} show that adding a temporal straightening loss \citep{henaff2019perceptual} reduces the curvature of the embeddings across time and hence enables better planning. SCALE \citep{hu2026scalestatecalibratedlatentembeddings} aligns latent distances with task-relevant state differences to make the planner-facing geometry more informative. \citet{wang2026adajepaadaptivelatentworld} also considers adapting the final layer features online to improve planning performance and achieves better out-of distribution generalization. It's also helpful to use inverse-dynamics modeling \citep{Lesort_2018}, temporal smoothness \citep{sobal2025learningrewardfreeofflinedata}, which can lead to better downstream planning performances \citep{terver2026lightweightlibraryenergybasedjointembedding}. 

\paragraph{Sparse and Discrete Representations in World Models.}
Discrete latent representations have previously been explored in world models. DreamerV2 \citep{hafner2022masteringataridiscreteworld} replaces Gaussian stochastic states with a product of categorical variables, which can equivalently be viewed as a structured sparse binary representation, and finds that categorical latents substantially improve performance over Gaussian ones. The authors hypothesize that sparsity may contribute to this improvement and that categorical states may better capture non-smooth transitions. Relatedly, Variational Sparse Gating \citep{jain2022learningrobustdynamicsvariational} introduces sparsity into latent dynamics by sparsely updating recurrent-state dimensions through stochastic binary gates. In contrast, we study sparsity directly as a geometry for continuous latent representations: LpWM learns distributed, non-negative sparse codes whose support is not fixed by a categorical grouping. Our focus is also different in that we ask whether sparsity reduces the complexity of the action-conditioned transition model itself, and whether the learned support recovers discrete dynamical regimes while the nonzero magnitudes retain continuous within-regime state.

\paragraph{Linearizing Nonlinear Dynamics.}

Koopman operator theory studies nonlinear autonomous systems through lifted observables with linear evolution, with extensions to controlled systems enabling linear prediction and MPC \citep{proctor2016dynamic,korda2018linear}. Recent work further characterizes the conditions for which exact finite-dimensional Koopman embeddings exist for controlled nonlinear systems \citep{shang2026existence}. Related approaches include SINDy, which identifies sparse nonlinear governing equations from candidate function libraries \citep{brunton2016discovering}. Data-driven feedback linearization also learns state and input transformations of yielding linear controlled dynamics \citep{goswami2023data}.

\begin{table}[t]
  \centering\small
  \setlength{\tabcolsep}{4.5pt}
  \caption{\textbf{Sparsity of the LpWM encoder representations across predictor types and latent dimensions.} Each entry is the fraction of non-zero coordinates $\mathbb{E}[\lVert\mathbf{z}\rVert_0]/D$ in the sparse code for the target distribution with $\{\mu=0,\sigma=\sqrt{1/2},p=1\}$, measured on the validation set at the tuned-best cell for the LpWM run paired with each predictor at latent dimension $D$. Dense LeWM has active fraction $1.0$ by construction. The learned code stays $\sim\!30$--$65\%$ active regardless of predictor or dimension, confirming the codes underlying the planning comparison are genuinely sparse.}
  \label{tab:l0-sparsity}
  \begin{tabular}{lcccccc}
    \toprule
    $D$ & \textbf{Deep-AdaLN}$(k)$ & \textbf{Shallow-AdaLN}$(k)$ & \textbf{MLP$\circ$LTV}$(k)$ & \textbf{MLP$\circ$LTI}$(k)$ & \textbf{LTI}$(k)$ & \textbf{LTI}$(1)$ \\
    \midrule
    $384$  & $0.34$ & $0.33$ & $0.37$ & $0.28$ & $0.41$ & $0.63$ \\
    $768$  & $0.44$ & $0.37$ & $0.54$ & $0.53$ & $0.41$ & $0.63$ \\
    $1536$ & $0.49$ & $0.39$ & $0.49$ & $0.54$ & $0.41$ & $0.60$ \\
    $2048$ & $0.42$ & $0.46$ & $0.51$ & $0.46$ & $0.43$ & $0.61$ \\
    $4096$ & $0.48$ & $0.49$ & $0.47$ & $0.43$ & $0.44$ & $0.56$ \\
    \bottomrule
  \end{tabular}
\end{table}

\begin{table}[h]
  \centering
  \small
  \caption{\textbf{Success Rate (\%) over 50 Test Samples in the Piecewise Environments with Varying Numbers of Zones}. Piecewise ($n\times n$) denotes the environment with varying numbers of zones. Goal images are either sampled from the evaluation dataset or randomly sampled from all possible locations of the dot agent. $H$ denotes the horizon and $R$ represents the receding-horizon. Results are shown as mean$\pm$std over 3 seeds. \textbf{Bold} denotes the best performance.}
  \label{tab:pw-sweep}
  \begin{tabular}{llccc}
    \toprule
    & & {Goals from Evaluation Set} & \multicolumn{2}{c}{Random Goals} \\
    \cmidrule(lr){3-3}\cmidrule(lr){4-5}
    Env & Method & $H=R=5$ & $H=5,R=1$ & $H=R=5$ \\
    \midrule
    Piecewise ($2\times 2$)
      & LpWM (Ours) & $\mathbf{99.33{\pm}1.15}$ & $\mathbf{84.67{\pm}4.16}$ & $59.33{\pm}2.31$ \\
      & LpWM+TJ (Ours) & $\mathbf{99.33{\pm}1.15}$ & $82.67{\pm}5.03$ & $\mathbf{68.67{\pm}4.62}$ \\
      & LeWM & $\mathbf{99.33{\pm}1.15}$ & $65.33{\pm}4.16$ & $36.00{\pm}3.46$ \\
    \midrule
    Piecewise ($3\times 3$)
      & LpWM (Ours) & $\mathbf{100.00{\pm}0.00}$ & $58.67{\pm}1.15$ & $56.00{\pm}4.00$ \\
      & LpWM+TJ (Ours) & $\mathbf{100.00{\pm}0.00}$ & $\mathbf{60.00{\pm}2.00}$ & $\mathbf{58.00{\pm}5.29}$ \\
      & LeWM & $\mathbf{100.00{\pm}0.00}$ & $50.00{\pm}5.29$ & $47.33{\pm}6.11$ \\
    \bottomrule
  \end{tabular}
\end{table}

\section{Predictor Designs}\label{app:preddesignsec}

\subsection{Overall Designs}

Table~\ref{tab:predictor-ladder} lists the predictors we compare, in order of decreasing complexity. \textbf{Deep-} and \textbf{Shallow-AdaLN}$(k)$ are $6$- and $1$-block Transformers with AdaLN-zero action conditioning \citep{peebles2023scalablediffusionmodelstransformers}; \textbf{LTI}$(k)$ and \textbf{LTI}$(1)$ are linear (time-invariant) state-space maps over $k$ and $1$ history lags; and \textbf{MLP$\circ$LTI}$(k)$ places a single nonlinear readout $\mathbf{W}\,\mathrm{ReLU}(\cdot)$ on top of the linear core. One caveat is that we use the $\operatorname{RepReLU}$ link function in LpWM when using \textbf{LTI}$(k)$ and \textbf{LTI}$(1)$ predictors, so we don't have exact linear dynamics models. 

We have also designed \textbf{MLP$\circ$LTV}$(k)$, a data-dependent (linear-time-varying) predictor that moves beyond the fixed weight matrices in the linear-time invariant case. 

\subsection{MLP$\circ$LTV$(k)$}

\textbf{MLP$\circ$LTV}$(k)$ retains the MLP$\circ$LTI form
$\hat{\mathbf{z}}_{t+1}=\sigma\!\big(\mathbf{W}\,\mathrm{ReLU}(\sum_{i=0}^{k-1}\mathbf{A}_i(\mathbf{z}_t)\,\mathbf{z}_{t-i}+\mathbf{B}(\mathbf{z}_t)\,\mathbf{a}_t)\big)$
but makes the operator matrices state-dependent. Naively, designing a linear map with a matrix-valued output given a vector-valued input will require a $O(D^3)$ tensor for the feature dimension $D$.

Instead, we consider a structured matrix parameterization, where the state emits only a small gate that modulates a fixed low-rank correction:
\begin{align}
\mathbf{A}_i(\mathbf{z}_t) \;=\; \mathbf{A}_i \;+\; \mathbf{U}_i\,\mathrm{diag}\!\big(\mathbf{g}_i(\mathbf{z}_t)\big)\,\mathbf{V}_i^{\top},
\qquad
\mathbf{g}(\mathbf{z}_t)=\mathrm{sigmoid}(\mathbf{G}\,\mathbf{z}_t)\in[0,1]^{r},
\end{align}
where $\mathbf{A}_i$ is a fixed base operator, $\mathbf{V}_i^{\top}\!:\mathbb{R}^{D}\!\to\!\mathbb{R}^{r}$ and $\mathbf{U}_i\!:\mathbb{R}^{r}\!\to\!\mathbb{R}^{D}$ are learned low-rank factors, and the gate $\mathbf{g}(\mathbf{z}_t)$---a single projection $\mathbf{G}\!:\mathbb{R}^{D}\!\to\!\mathbb{R}^{(k+1)r}$ shared across the $k$ lags and the action term---selects among $r$ fixed low-rank ``modes'' ($r{=}16$ in our experiments). The action operator $\mathbf{B}(\mathbf{z}_t)$ is gated identically. This costs $O(Dr)$ with no $D^3$ term. Crucially, only $\mathbf{z}_t$ drives the gate, so the action $\mathbf{a}_t$ still enters affinely: the dynamics core stays control-affine up to the shared nonlinear readout. 

We zero-initialize the up-projections $\mathbf{U}_i$ (LoRA-style \citep{hu2021loralowrankadaptationlarge}) as an override after $\mu$P initialization, so at initialization every correction vanishes and MLP$\circ$LTV$(k)$ coincides exactly with MLP$\circ$LTI$(k)$. It is therefore a strict superset that recovers the fixed-operator model unless the data reward state-dependence. 

\begin{figure}[t]
  \centering
  \begin{subfigure}[b]{0.24\linewidth}
    \includegraphics[width=\linewidth]{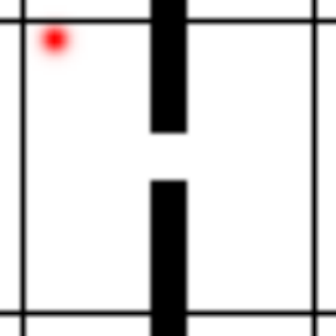}
    \caption{Wall}\label{fig:env-wall}
  \end{subfigure}\hfill
  \begin{subfigure}[b]{0.24\linewidth}
    \includegraphics[width=\linewidth]{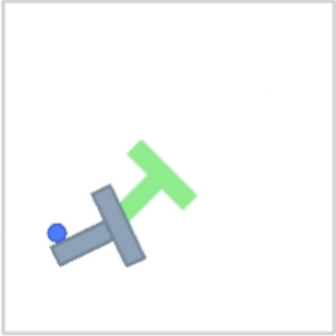}
    \caption{PushT}\label{fig:env-pusht}
  \end{subfigure}\hfill
  \begin{subfigure}[b]{0.24\linewidth}
    \includegraphics[width=\linewidth]{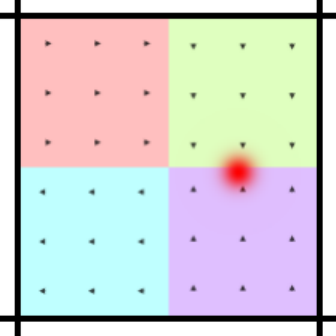}
    \caption{Piecewise}\label{fig:env-piecewise}
  \end{subfigure}\hfill
  \begin{subfigure}[b]{0.24\linewidth}
    \includegraphics[width=\linewidth]{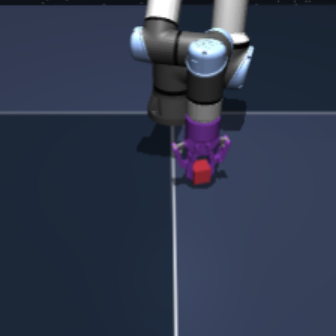}
    \caption{OGBench-Cube}\label{fig:env-cube}
  \end{subfigure}
  \caption{\textbf{Environments.} \textbf{(a)} Wall: 2D navigation through a doorway (red agent). \textbf{(b)} PushT: push a T-shaped block to a target pose (green). \textbf{(c)} Piecewise: an open room partitioned into force-field zones (colored) with a known ground-truth segmentation. \textbf{(d)} OGBench-Cube: a robot arm manipulating a cube.}
  \label{fig:environments}
\end{figure}

\section{Environment}\label{app:envexplained}

\subsection{TwoRoom, PushT, and OGBench Cube}

We use the saved dataset in \citet{zhou2025dinowmworldmodelspretrained} for TwoRoom, PushT, and in \citet{maes2026stableworldmodelplatformreproducibleworld} for OGBench Cube. During planning, we fix the goal location to be $25$ steps away from the starting location and select episodes from the evaluation set. The action block size is fixed to be $5$.

\subsection{Piecewise}\label{app:piecewiseenvintro}

\paragraph{Piecewise.} To measure what the sparse code encodes about the
dynamics, we use a synthetic, analytically-known environment called Piecewise \citep{maes2026leworldmodelstableendtoendjointembedding}. An open
$224\times224$ room (no walls) is divided into a $g\times g$ grid of zones. Each zone $i$ applies a fixed additive bias to motion,
$\mathbf{p}_{t+1}=\mathbf{p}_t+\mathbf{a}_t\cdot\text{speed}+\mathbf{b}[\text{zone}(\mathbf{p}_t)]$, with radial biases
$\mathbf{b}_i=\text{bias\_scale}\cdot(\cos\theta_i,\sin\theta_i)$,
$\theta_i=2\pi i/g^2$. The dynamics are thus locally linear but globally
piecewise-affine, with a known ground-truth zone segmentation and an exact
expert policy (it inverts the motion equation). This is a controlled
environment where the
latent mode is known, so we can test directly whether the sparse support
recovers it. We use $g\in\{2,3\}$ (i.e. 4 and 9 zones) in our experiments.

\subsection{Visualization of Environments}

See Figure~\ref{fig:environments} for example snapshots of the environments. Note that there is another version of the Piecewise environment without the background color rendering as we have shown in Figure~\ref{fig:aaa1}. 

\section{Experimental Details}
\label{app:experimental-details}

Our experiments comprise two complementary studies implemented in separate experimental pipelines. Section~\ref{sec:doessparsitysimplify} systematically varies predictor complexity and latent dimension to test whether sparsity reduces the complexity of the learned dynamics in Wall and PushT environments. The interpretability study in Section~\ref{sec:sparseforinterpretability} instead fixes the model configuration and probes the structure of the learned sparse support in the Piecewise and OGBench-Cube environments. The two pipelines instantiate the same underlying LpWM formulation but differ in architecture and optimization details, summarized in Table~\ref{tab:exp-details}. All comparisons and ablations are performed within the corresponding experimental pipeline.

In both pipelines, we follow the settings in Section~\ref{sec:method-rdm} and Section~\ref{sec:method-arch} for the loss, architecture, and link function designs. The model is an action-conditioned JEPA trained end-to-end. The encoder and predictor outputs pass through a link function $\sigma$, and per-timestep latent marginals are matched by 2-Wasserstein distance to a target distribution obtained by applying the same link to a Generalized Gaussian with shape parameter $p$ and location or mean-shift parameter $\mu$. Dense LeWM uses $\sigma=\mathrm{identity}$ with an isotropic Gaussian target ($p=2$), whereas sparse LpWM uses $\sigma=\mathrm{(Rep)ReLU}$ with a Rectified Laplace target ($p=1$), yielding sparse representations.

\begin{table}[t]
  \centering
  \small
  \setlength{\tabcolsep}{4.5pt}
  \caption{\textbf{Experimental configurations for Section~\ref{sec:doessparsitysimplify} and Section~\ref{sec:sparseforinterpretability}.}
  We list out the experimental pipelines in the following table.}
  \label{tab:exp-details}
  \begin{tabular}{lcccc}
    \toprule
    & \multicolumn{2}{c}{\textbf{Section~\ref{sec:doessparsitysimplify}}} 
    & \multicolumn{2}{c}{\textbf{Section~\ref{sec:sparseforinterpretability}}} \\
    \cmidrule(lr){2-3}\cmidrule(lr){4-5}
    & \textbf{Wall} & \textbf{PushT} & \textbf{Piecewise} & \textbf{OGBench-Cube} \\
    \midrule
    Encoder 
      & \multicolumn{2}{c}{from-scratch ViT, $12$ layers, width $384$}
      & \multicolumn{2}{c}{from-scratch ViT-Tiny, width $192$} \\
    Predictor
      & \multicolumn{2}{c}{See Table~\ref{tab:predictor-ladder}}
      & \multicolumn{2}{c}{Deep-AdaLN$(k)$} \\
    Dimension ($D$) & \multicolumn{2}{c}{$\{384,768,1536,2048,4096\}$} & \multicolumn{2}{c}{$192$} \\
    History $k$
      & $1$ & $3$ & $3$ & $3$ \\
    Frameskip
      & $5$ & $5$ & $5$ & $5$ \\
    Batch size
      & $128$ & $64$ & $128$ & $128$ \\
    Epochs
      & $20$ & $2$ & $10$ & $10$ \\
    Optimizer
      & \multicolumn{2}{c}{Adam, $\mu$P (with base width $384$)}
      & \multicolumn{2}{c}{AdamW} \\
    Base LR
      & \multicolumn{2}{c}{See Appendix~\ref{sec:thoroughhyperforpusht}}
      & \multicolumn{2}{c}{$5\!\times\!10^{-5}$} \\
    Grad.\ clip
      & \multicolumn{2}{c}{0.0}
      & \multicolumn{2}{c}{$1.0$} \\
    Reg.\ weight $\lambda$
      & \multicolumn{2}{c}{See Appendix~\ref{sec:thoroughhyperforpusht}}
      & $25$ & $10$ \\
    Number of projections
      & \multicolumn{2}{c}{$8192$}
      & \multicolumn{2}{c}{$1024$} \\
    Parameter $\mu$
      & \multicolumn{2}{c}{$\{0,-1,-2\}$}
      & $0$ & $0$ \\
    Temporal Jaccard
      & off & off & $\{0.0,0.01\}$ & $\{0.0,0.005,0.01,0.05, 0.1\}$ \\
    \bottomrule
  \end{tabular}
\end{table}

\textbf{Experimental Details for Section~\ref{sec:doessparsitysimplify}.} We follow the settings in DINO-WM \citep{zhou2025dinowmworldmodelspretrained} and set the context length ($k$) for the Wall environment to be $k=1$. PushT uses $k=3$ because velocity is not recoverable from a single observation. For each predictor type and latent feature dimension configuration, we sweep the RDMReg weight $\lambda_{\operatorname{RDMReg}}$ and the $\mu$P base learning rate $\ell_{\operatorname{muP}}$(Appendix~\ref{sec:thoroughhyperforpusht}) and report the best-performing cell in the corresponding grid. We also swept the hyperparameter for LeWM and its VICReg variants in Appendix~\ref{sec:thoroughhyperforpusht} to make sure we obtain strong baseline performances. 

\textbf{Experimental Details for Section~\ref{sec:sparseforinterpretability}.} We follow the settings in stable-worldmodel \citep{maes2026stableworldmodelplatformreproducibleworld}. For both the Piecewise and OGBench-Cube environments, we use a from-scratch ViT-Tiny encoder (width $192$, CLS token), latent dimension $D=192$, and a depth-$6$ Deep-AdaLN$(k)$ predictor with history $k=3$. Models are trained for $10$ epochs using AdamW with learning rate $5\times10^{-5}$, weight decay $10^{-3}$, a linear-warmup cosine schedule with warmup over approximately $1\%$ of the optimization steps, batch size $128$, bf16 precision, and gradient clipping at $1.0$. These experiments do not use $\mu$P since the feature dimension is relatively small, i.e. $D=192$. 

\paragraph{Planning and Evaluation.}

All experiments use the cross-entropy method (CEM) to optimize action sequences in latent space by minimizing the MSE between predicted latent rollouts and the encoded goal representation. Both experimental pipelines use $300$ candidate action sequences per CEM iteration, retain the top $30$ best-performing action sequences, and run CEM for $30$ iterations with an initial variance scale of $1.0$ and planning horizon $H=5$. With a frameskip of $5$, the goal observation is therefore $25$ raw environment steps ahead.

In Section~\ref{sec:doessparsitysimplify}, we report both open-loop and closed-loop planning success rates. Open-loop evaluation executes a single CEM plan without replanning and therefore provides a more direct measure of accumulated model-prediction quality. Closed-loop evaluation uses receding-horizon MPC, replanning after each block of $5$ executed actions, with at most $10$ replanning iterations. Wall and PushT are evaluated over $50$ trajectories whose goals are drawn from the evaluation dataset.

In Section~\ref{sec:sparseforinterpretability}, we additionally evaluate Piecewise using randomly sampled goal images, since performance for both LeWM and LpWM saturates when goals are sampled from the evaluation set. In Table~\ref{tab:pw-sweep}, we also consider a receding horizon of $1$. For OGBench-Cube, we follow the evaluation protocol of \citet{maes2026leworldmodelstableendtoendjointembedding}, where goal observations are sampled only $25$ raw environment steps ahead. This short horizon can make the planning task relatively easy, potentially leading to performance saturation and limiting the discriminative power of the benchmark. We leave evaluation under more challenging longer-horizon and larger-frameskip settings to future work.

\section{Jaccard Index}\label{app:jaccardlossessec}

\subsection{Hard and Soft Jaccard index}
\label{app:soft-jaccard}

The standard Jaccard index
between two sets $A$ and $B$ is $\operatorname{Jac}(A,B)=\vert A\cap B\vert/\vert
A\cup B\vert$. For two binary vectors $\mathbf{x},\mathbf{y}\in\{0,1\}^D$ this is
$J(\mathbf{x},\mathbf{y})=\sum_{i=1}^{D}\mathbbm{1}_{\mathbf{x}_i=1\land\mathbf{y}_i=1}\,/\,\sum_{i=1}^{D}\mathbbm{1}_{\mathbf{x}_i=1\lor\mathbf{y}_i=1}$,
where $\land,\lor$ are logical ``and'', ``or'' symbols. Because the encoder's output is
non-negative but not binary, directly optimizing this set Jaccard is
non-differentiable. We therefore use the soft relaxation \citep{ruzicka1958geobotanik}, for $\mathbf{x},\mathbf{y}\in\mathbb{R}^{D}_{\geq 0}$,
\begin{align}
    J_{S}(\mathbf{x},\mathbf{y})=\frac{\sum_{i=1}^{D}\min(\mathbf{x}_i,\mathbf{y}_i)}{\sum_{i=1}^{D}\max(\mathbf{x}_i,\mathbf{y}_i)+\epsilon},
\end{align}
where $\min(\cdot,\cdot)$ and $\max(\cdot,\cdot)$ are smooth proxies for $\land$
and $\lor$; a higher $J_S$ means more similar supports, and $J_S$ recovers the
exact set Jaccard when its arguments are binary.

\subsection{Temporal Jaccard Loss}\label{sec:temp-jaccard-loss}

Enforcing the feature distribution to be sparse independently across time steps via RDMReg imposes no constraints on the desired temporal transition. Therefore we also consider the optional temporal Jaccard loss (TJ) as a slowness prior on the binary support vector. The temporal Jaccard loss penalizes rapid support transitions between consecutive frames and hence encourages the support to be smoothly varying across time while not interfering with the global sparsity constraints brought by RDMReg. 

\textbf{Temporal Jaccard loss.} Given encoder embeddings $\mathbf{Z} \in \mathbb{R}_{\geq 0}^{B \times T \times D}$, where $B$ is batch size, $T$ is the temporal dimension, and $D$ is the feature dimension, the temporal jaccard loss is the mean support instability over all consecutive frame pairs:
\begin{align}
    \mathcal{L}_{\operatorname{TJ}}=\frac{1}{B(T-1)} \sum_{b=1}^{B} \sum_{t=1}^{T-1} \Big(1 - J_{S}\big(\mathbf{Z}[b,t,:], \mathbf{Z}[b,t+1,:]\big)\Big)
\end{align}
where $\mathbf{Z}[b,t,:]\in\mathbb{R}^D$ denotes the slice of the embedding tensor at batch index $b$ and time index $t$. Minimizing $\mathcal{L}_{\operatorname{TJ}}$ encourages the support between embeddings of consecutive frames to be smoothly varying, and hence reduces the chance of drastic support changes between temporally adjacent frames. 

\section{Maximal-Update Parameterization (\texorpdfstring{$\mu$P}{muP}) for JEPA}
\label{app:mupforjepa}

Our predictor-complexity experiments vary the latent dimensions over more than an order of magnitude, from $D=384$ to $D=4096$. A naive parameterization uses one learning rate for all layers in the neural network as $D$ varies, so performance differences across widths can be confounded by width-specific optimization effects rather than properties of the learned representation. We therefore use maximal-update parameterization ($\mu$P) \citep{yang2022tensorprogramsvtuning} to preserve feature-learning dynamics across width and enable learning-rate transfer.

For a dense linear map $\mathbf{h}=\mathbf{W}\mathbf{x}$ with $\mathbf{W}\in\mathbb{R}^{d_{\mathrm{out}}\times d_{\mathrm{in}}}$, $\mu$P chooses initialization and learning-rate scalings such that both the hidden activations and their training-induced updates remain of order $\Theta(\sqrt{d_{\mathrm{out}}})$ as width changes. Equivalently, the weight and update operators remain at the appropriate spectral scale for stable feature learning \citep{yang2024spectralconditionfeaturelearning}. For Adam-type optimizers, this gives an effective learning-rate scaling proportional to the inverse fan-in,
\begin{align}
\eta_{\mathbf{W}} \propto \frac{1}{d_{\mathrm{in}}}.
\end{align}
We apply this rule to all width-dependent dense maps in the JEPA encoder and predictor. Let $D_0$ denote the reference latent width at which a base learning rate $\eta_{\mathrm{base}}$ is selected. In our experiment, we fix $D_0=384$. For parameter matrices whose fan-in scales linearly with the latent dimension $D$, we use the following scaling equation:
\begin{align}
\eta_{\mathbf{W}}(D)
=
\eta_{\mathrm{base}}
\frac{D_0}{D}.
\end{align}
Width-independent parameter groups retain their base learning rate. Initialization follows the corresponding $\mu$P fan-in scaling. For all other layers or components like layer norm or bias, we follow the design recipe in Appendix B.1 in \citet{yang2022tensorprogramsvtuning}.

Given that $\mu$P has not been demonstrated in JEPA architectures before for hyperparameter transfer across model width to our best knowledge, we don't intend to use $\mu$P to reduce the hyperparameter tuning efforts. Instead, we use $\mu$P here simply to ensure stable feature learning, i.e. maintaining a stable spectral norm of the weight matrices to be around $\|\mathbf{W}\|_{*}=\Theta(\sqrt{d_{\mathrm{out}}/d_{\mathrm{in}}})$.

In Appendix~\ref{sec:thoroughhyperforpusht}, we show the $\mu$P parameterization of learning rate can transfer across width most of the time, although there are occasional exceptions. We intend to investigate this more in future work.

\section{Proofs of Linearization}
\label{app:sparsity-linearity}

\subsection{Proof of Proposition~\ref{prop:linearapproxofnonlinear}}
\label{app:sl-theory}

\begin{proof}\label{proof:linapproxofnonlinearwithsparse}

Since $\mathcal{X}$ is a compact subset of $\mathbb{R}^d$, the open cover of $\mathcal{X}$
\begin{align}
\mathcal{X}
\subseteq
\bigcup_{\mathbf{c}\in\mathcal{X}}
B_\varepsilon(\mathbf{c})
\end{align}
contains a finite subcover. In other words, for every $\varepsilon>0$, there exists $C_\varepsilon=\{\mathbf{c}_1,\ldots,\mathbf{c}_N\}
\subset\mathcal{X}$ such that
\begin{align}
\mathcal{X}
\subseteq
\bigcup_{\mathbf{c}\in C_\varepsilon}
B_\varepsilon(\mathbf{c})
\end{align}
where $B_\varepsilon(\mathbf{c})
:=
\left\{
\mathbf{x}\in\mathcal{X}
:
\|\mathbf{x}-\mathbf{c}\|
<
\varepsilon
\right\}.$ Now consider a deterministic quantization map $Q_\varepsilon:
\mathcal{X}\rightarrow\{1,\ldots,N\}$ such that $\left\|
\mathbf{x}
-
\mathbf{c}_{Q_\varepsilon(\mathbf{x})}
\right\|
\le
\varepsilon$ for all $\mathbf{x}\in\mathcal{X}$.

Let's define the encoder as $E_\varepsilon(\mathbf{x})
=
\mathbf{e}_{Q_\varepsilon(\mathbf{x})},$ where $\mathbf{e}_i$ is the $i$-th standard basis vector of
$\mathbb{R}^N$. Hence $\|E_\varepsilon(\mathbf{x})\|_0=1$ for every $\mathbf{x}\in\mathcal{X}$, i.e. $E_\varepsilon(\cdot)$ is a one-hot encoder. 

Correspondingly, we also define the decoder by $D_\varepsilon(\mathbf{e}_i)=\mathbf{c}_i.$ Now fix any action $\mathbf{a}\in\mathcal{A}$. For each representative
$\mathbf{c}_i$, define $\tau_{\mathbf{a}}(i)
:=
Q_\varepsilon\!\left(
f(\mathbf{c}_i,\mathbf{a})
\right)$ as the transition dynamics map.
Then the transition matrix $\mathbf{P}_\varepsilon(\mathbf{a})$ can be defined column-wise by $\mathbf{P}_\varepsilon(\mathbf{a})\mathbf{e}_i
=
\mathbf{e}_{\tau_{\mathbf{a}}(i)}.$

Thus every column of $\mathbf{P}_\varepsilon(\mathbf{a})$ contains exactly
one nonzero entry, and
\begin{align}
\mathbf{z}_{t+1}
=
\mathbf{P}_\varepsilon(\mathbf{a}_t)\mathbf{z}_t
\end{align}
is exactly linear in $\mathbf{z}$. It remains to bound the decoded one-step error. Fix arbitrary $\mathbf{x}_t\in\mathcal{X}$ and
$\mathbf{a}_t\in\mathcal{A}$, and let $i:=Q_\varepsilon(\mathbf{x}_t).$ Then by the definition of the encoder we have $E_\varepsilon(\mathbf{x}_t)=\mathbf{e}_i$ with $\|\mathbf{x}_t-\mathbf{c}_i\|
\le
\varepsilon.$
We can rollout the latent dynamics to obtain 
\begin{align}
\mathbf{P}_\varepsilon(\mathbf{a}_t)E_\varepsilon(\mathbf{x}_t)
&=
\mathbf{P}_\varepsilon(\mathbf{a}_t)\mathbf{e}_i
=
\mathbf{e}_{\tau_{\mathbf{a}_t}(i)}
\end{align}
The decoded state is thus
\begin{align}
\hat{\mathbf{x}}_{t+1}
&:=
D_\varepsilon\!\left(
\mathbf{P}_\varepsilon(\mathbf{a}_t)
E_\varepsilon(\mathbf{x}_t)
\right) =
\mathbf{c}_{
Q_\varepsilon(f(\mathbf{c}_i,\mathbf{a}_t))
}.
\end{align}
Let the ground-truth next state be 
\begin{align}
\mathbf{x}_{t+1}
=
f(\mathbf{x}_t,\mathbf{a}_t).
\end{align}
By the triangle inequality, we have
\begin{align}
\|\mathbf{x}_{t+1}-\hat{\mathbf{x}}_{t+1}\|
&=
\left\|
f(\mathbf{x}_{t},\mathbf{a}_{t})
-
\mathbf{c}_{
Q_\varepsilon(f(\mathbf{c}_i,\mathbf{a}_{t}))
}
\right\| \\ &\le
\left\|
f(\mathbf{x}_{t},\mathbf{a}_{t})
-
f(\mathbf{c}_i,\mathbf{a}_{t})
\right\|+
\left\|
f(\mathbf{c}_i,\mathbf{a}_{t})
-
\mathbf{c}_{
Q_\varepsilon(f(\mathbf{c}_i,\mathbf{a}_{t}))
}
\right\|.
\end{align}
Since $f$ is Lipschitz continuous in state, we know 
\begin{align}
\left\|
f(\mathbf{x}_{t},\mathbf{a}_{t})
-
f(\mathbf{c}_i,\mathbf{a}_{t})
\right\|
&\le
L\|\mathbf{x}_{t}-\mathbf{c}_i\| \le
L\varepsilon.
\end{align}
Now because $f(\mathbf{c}_i,\mathbf{a}_{t})\in\mathcal{X}$,
the compactness property implies that
\begin{align}
\left\|
f(\mathbf{c}_i,\mathbf{a}_t)
-
\mathbf{c}_{
Q_\varepsilon(f(\mathbf{c}_i,\mathbf{a}_t))
}
\right\|
\le
\varepsilon.
\end{align}
Therefore the approximation error is upper bounded by
\begin{align}
\|\mathbf{x}_{t+1}-\hat{\mathbf{x}}_{t+1}\|
&\le
L\varepsilon+\varepsilon=
(L+1)\varepsilon.
\end{align}

Since the bound is uniform over $\mathbf{x}_t$ and $\mathbf{a}_t$,
\begin{align}
\sup_{\mathbf{x}_t\in\mathcal{X},\,\mathbf{a}_t\in\mathcal{A}}
\left\|
f(\mathbf{x}_t,\mathbf{a}_t)
-
D_\varepsilon\!\left(
\mathbf{P}_\varepsilon(\mathbf{a}_t)
E_\varepsilon(\mathbf{x}_t)
\right)
\right\|
\le
(L+1)\varepsilon.
\end{align}
Now since the Lipschitz constant $L$ is fixed, we know that $(L+1)\varepsilon\rightarrow0$ as $\varepsilon\rightarrow0$. Therefore, we arrive at
\begin{align}
\lim_{\varepsilon\rightarrow0}
\sup_{\mathbf{x}_t\in\mathcal{X},\,\mathbf{a}_t\in\mathcal{A}}
\left\|
f(\mathbf{x}_t,\mathbf{a}_t)
-
D_\varepsilon\!\left(
\mathbf{P}_\varepsilon(\mathbf{a}_t)
E_\varepsilon(\mathbf{x}_t)
\right)
\right\|
=
0.
\end{align}
\end{proof}

\begin{lemma}[Finite-horizon rollout bound]\label{lemma:finiterollout}
Under the assumptions of Proposition~\ref{prop:linearapproxofnonlinear}, let $\mathbf{x}_{t+1}
=
f(\mathbf{x}_t,\mathbf{a}_t)$ for $t=0,\ldots,H-1$ be the ground-truth trajectory under an arbitrary action sequence
$\mathbf{a}_0,\ldots,\mathbf{a}_{H-1}\in\mathcal{A}$. Initialize $\mathbf{z}_0=E_\varepsilon(\mathbf{x}_0)$ and recursively define the latent transition as $\mathbf{z}_{t+1}=\mathbf{P}_\varepsilon(\mathbf{a}_t)\mathbf{z}_t.$
Let $\hat{\mathbf{x}}_t=D_\varepsilon(\mathbf{z}_t)$ be the decoded state. 
Then
\begin{align}
\|\mathbf{x}_H-\hat{\mathbf{x}}_H\|
\le
\varepsilon
\sum_{k=0}^{H}L^k.
\end{align}
\end{lemma}

\begin{proof}
Define the decoded quantization map as $q_\varepsilon(\mathbf{x})
:=
D_\varepsilon(E_\varepsilon(\mathbf{x}))
=
\mathbf{c}_{Q_\varepsilon(\mathbf{x})}.$ By construction, $\|\mathbf{x}-q_\varepsilon(\mathbf{x})\|
\le
\varepsilon$ for all $\mathbf{x}\in\mathcal{X}$. Since every latent state $\mathbf{z}_t$ is one-hot, we know that if $\mathbf{z}_t=\mathbf{e}_i,$ then $\hat{\mathbf{x}}_t
=
D_\varepsilon(\mathbf{z}_t)
=
\mathbf{c}_i.$
Since the latent dynamics give
\begin{align}
\mathbf{z}_{t+1}
&=
\mathbf{P}_\varepsilon(\mathbf{a}_t)\mathbf{e}_i=
\mathbf{e}_{
Q_\varepsilon(f(\mathbf{c}_i,\mathbf{a}_t))
}.
\end{align}
We know that 
\begin{align}
\hat{\mathbf{x}}_{t+1}
&=
D_\varepsilon(\mathbf{z}_{t+1})=
q_\varepsilon\!\left(
f(\hat{\mathbf{x}}_t,\mathbf{a}_t)
\right).
\end{align}
Define the scalar rollout error as $e_t
:=
\|\mathbf{x}_t-\hat{\mathbf{x}}_t\|.$ Then
\begin{align}
e_{t+1}
&=
\left\|
f(\mathbf{x}_t,\mathbf{a}_t)
-
q_\varepsilon\!\left(
f(\hat{\mathbf{x}}_t,\mathbf{a}_t)
\right)
\right\| \\ &\le
\left\|
f(\mathbf{x}_t,\mathbf{a}_t)
-
f(\hat{\mathbf{x}}_t,\mathbf{a}_t)
\right\| \nonumber+
\left\|
f(\hat{\mathbf{x}}_t,\mathbf{a}_t)
-
q_\varepsilon\!\left(
f(\hat{\mathbf{x}}_t,\mathbf{a}_t)
\right)
\right\| \\
&\le
Le_t+\varepsilon.
\end{align}

At $t=0$, $e_0
=
\|\mathbf{x}_0-\hat{\mathbf{x}}_0\|
\le
\varepsilon.$ Rolling out $H$ steps give us
\begin{align}
e_H
&\le
L^He_0
+
\varepsilon\sum_{k=0}^{H-1}L^k \le
L^H\varepsilon
+
\varepsilon\sum_{k=0}^{H-1}L^k =
\varepsilon\sum_{k=0}^{H}L^k.
\end{align}
Hence we arrive at our conclusion
\begin{align}
\|\mathbf{x}_H-\hat{\mathbf{x}}_H\|
\le
\varepsilon\sum_{k=0}^{H}L^k.
\end{align}
\end{proof}

\subsection{Proof of Corollary~\ref{coro:sparsityinhighd}}
\begin{proof}\label{proof:sparsityinhighdimproof}
Partition each coordinate interval $[0,1]$ into $n$ subintervals of equal
length $h=\frac{1}{n}.$ The resulting partition of $[0,1]^d$ consists of $n^d=N$ axis-aligned
hypercubes, each with side length $h$.

Let $\mathbf{c}_i$ denote the center of one such hypercube. For any point
$\mathbf{x}$ in that cell, each coordinate differs from the corresponding
coordinate of $\mathbf{c}_i$ by at most $h/2$. Hence
\begin{align}
|\mathbf{x}_j-\mathbf{c}_{i,j}|
\le
\frac{h}{2},
\qquad
j=1,\ldots,d.
\end{align}
Therefore, using the Euclidean norm,
\begin{align}
\|\mathbf{x}-\mathbf{c}_i\|^2
&=
\sum_{j=1}^{d}
|\mathbf{x}_j-\mathbf{c}_{i,j}|^2 \le
\sum_{j=1}^{d}
\left(\frac{h}{2}\right)^2 =
d\frac{h^2}{4}.
\end{align}
Taking square roots gives $\|\mathbf{x}-\mathbf{c}_i\|
\le
\frac{\sqrt{d}}{2}h.$ Since $h=1/n$, we arrive at
\begin{align}
\varepsilon_N
=
\frac{\sqrt{d}}{2n}=\frac{\sqrt{d}}{2}N^{-1/d}.
\end{align}

Based on Proposition~\ref{prop:linearapproxofnonlinear}, the one-step decoded prediction error is bounded by
\begin{align}
\sup_{\mathbf{x},\mathbf{a}}
\left\|
f(\mathbf{x},\mathbf{a})
-
D_{\varepsilon_N}\!\left(
\mathbf{P}_{\varepsilon_N}(\mathbf{a})E_{\varepsilon_N}(\mathbf{x})
\right)
\right\|
\le
(L+1)\varepsilon_N.
\end{align}
Substituting the expression for $\varepsilon_N$ gives
\begin{align}
\sup_{\mathbf{x},\mathbf{a}}
\left\|
f(\mathbf{x},\mathbf{a})
-
D_{\varepsilon_N}\!\left(
\mathbf{P}_{\varepsilon_N}(\mathbf{a})E_{\varepsilon_N}(\mathbf{x})
\right)
\right\|
&\le
(L+1)
\frac{\sqrt{d}}{2}N^{-1/d} =
\frac{\sqrt{d}}{2}(L+1)N^{-1/d}.
\end{align}

Since $d$ and $L$ are fixed constants,
\begin{align}
\sup_{\mathbf{x},\mathbf{a}}
\left\|
f(\mathbf{x},\mathbf{a})
-
D_{\varepsilon_N}\!\left(
\mathbf{P}_{\varepsilon_N}(\mathbf{a})E_{\varepsilon_N}(\mathbf{x})
\right)
\right\|
=
O(N^{-1/d}).
\end{align}
Similarly, the finite-horizon rollout bound from Lemma~\ref{lemma:finiterollout} gives
\begin{align}
\|\mathbf{x}_H-\hat{\mathbf{x}}_H\|
\le
\varepsilon_N
\sum_{k=0}^{H}L^k.
\end{align}
Substituting the covering radius, we have
\begin{align}
\|\mathbf{x}_H-\hat{\mathbf{x}}_H\|
\le
\frac{\sqrt{d}}{2}
N^{-1/d}
\sum_{k=0}^{H}L^k.
\end{align}

Since $H$, $d$, and $L$ are fixed,
it follows that
\begin{align}
\lim_{N\rightarrow\infty}
\|\mathbf{x}_H-\hat{\mathbf{x}}_H\|
=
0.
\end{align}
\end{proof}

\section{Hyperparameter Sweeps}\label{sec:thoroughhyperforpusht}

\subsection{LpWM and LeWM Sweeps in the PushT Environment}\label{subsec:lpwmvslewmsweepinpusht}

To ensure full reproducibility and transparency, we perform and report hyperparameter sweeps underlying Figure~\ref{fig:sparsitylinearityfigsub22}. 

We run a grid search over the RDMReg regularization weight $\lambda_{\operatorname{RDMReg}}\in\{0.01,0.1,1.0,10.0\}$ and the muP (or $\mu$P) learning rate $\ell_{\operatorname{muP}}\in\{0.05,0.005,0.0005,0.00005\}$ for each fixed feature dimension $d\in\{384,768, 1536, 2048, 4096\}$, method (LeWM and LpWM with target distribution $\prod_{i=1}^d\operatorname{ReLU}(\operatorname{Laplace}(\mu=0,\sigma=\sqrt{1/2}))$), and predictor setting in Table~\ref{tab:predictor-ladder}. To ensure fair comparison and standardized hyperparameter sweeps, we use our sliced Wasserstein distance formulation as the distribution-matching loss for isotropic Gaussian targets instead of the Epps-Pulley loss in SIGReg \citep{balestriero2025lejepaprovablescalableselfsupervised}. They are both in fact instantiations of Cramer-Wold theorem, and only differ in one-sample vs. two-sample style distribution matching \citep{kuang2026rectifiedlpjepajointembeddingpredictive}.  

Due to compute constraints, we fix one seed for training but average over $3$ seeds during planning evaluations with Cross-Entropy Methods (CEM) \citep{rubinstein1999cross}. In practice, we execute our sweep sequentially. So there are some predictor types where we have done a very thorough grid search. After we have determined that some hyperparameters are always leading to bad performances (e.g. usually learning rate like $0.05$ is too large), we drop them in our subsequent sweeps so the resulting figures for different predictor types usually cover different subset of the full grid. 

In Figure~\ref{fig:pushtsweepmlpltik}, we observe that $\ell_{\operatorname{muP}}\in\{0.05,0.005\}$ almost always leads to $0\%$ success rate. Subsequent hyperparameter sweeps shown in Figure~\ref{fig:pushtsweepdeepadalnk}, Figure~\ref{fig:pushtsweepshallowadalnk}, Figure~\ref{fig:pushtsweepmlpltvk}, Figure~\ref{fig:pushtsweepltik}, and Figure~\ref{fig:pushtsweeplti1} selectively remove portions of the grid to save compute. 

\subsection{VICReg Sweep in the PushT Environment}\label{subsec:vicregsweepinpusht}

We also perform extensive hyperparameter sweeps for JEPA regularized with VICReg. Basically, we take the LeWM setting and replaced the SIGReg regularizer with VICReg. The resulting model doesn't have a name yet, and we abuse the notation and simply call it VICReg here. 

In Figure~\ref{fig:vicpushtsweepdeepadalnk}, Figure~\ref{fig:vicpushtsweepmlpltvk}, and Figure~\ref{fig:vicpushtsweepmlpltik}, we show the sweep results for the predictors Deep-AdaLN$(k)$, MLP$\circ$LTV$(k)$, and MLP$\circ$LTI$(k)$ respectively.

\begin{figure*}[t!]
    \centering

    \begin{subfigure}[t]{0.95\linewidth}
        \centering
        \includegraphics[width=\linewidth]{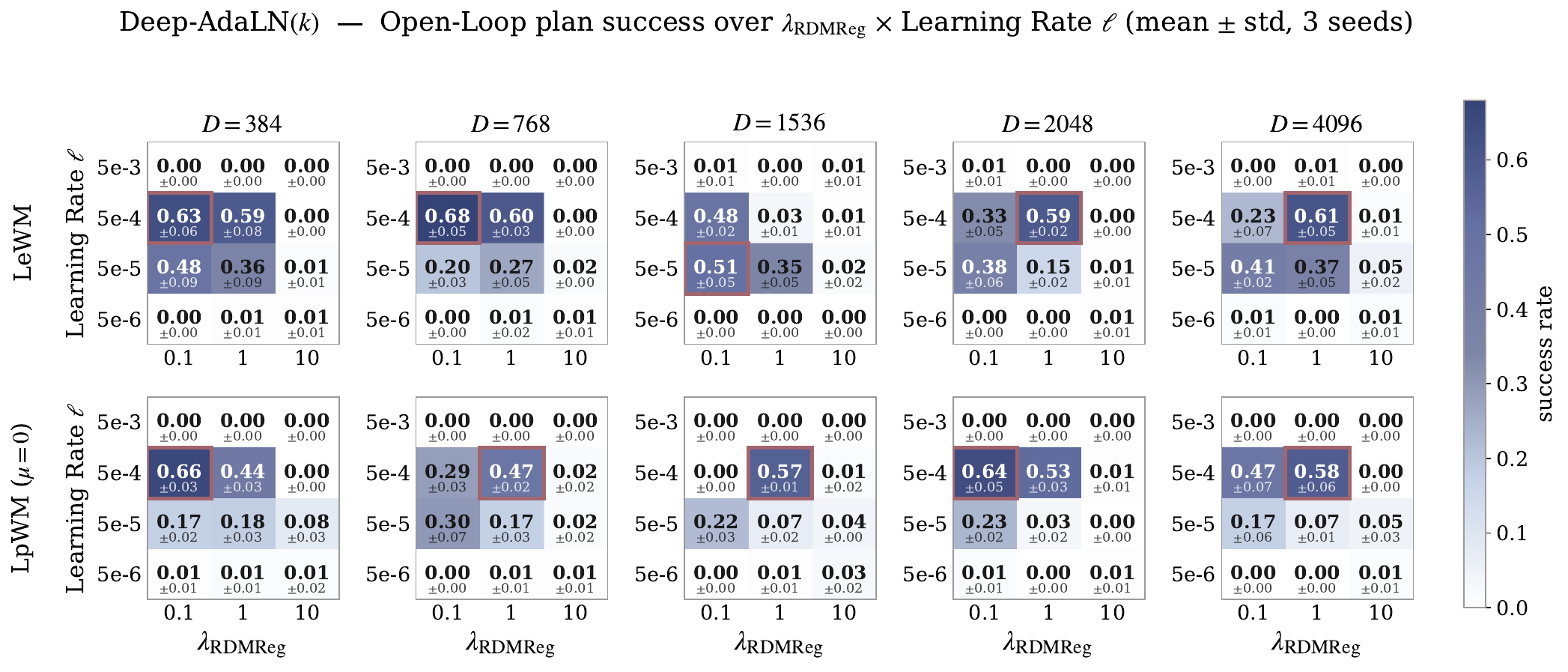}
        \captionsetup{justification=centering}
        \caption{Open-Loop.}
        \label{fig:pushtsweepdeepadalnkopen}
    \end{subfigure}

    \vspace{0.5cm}

    \begin{subfigure}[t]{0.95\linewidth}
        \centering
        \includegraphics[width=\linewidth]{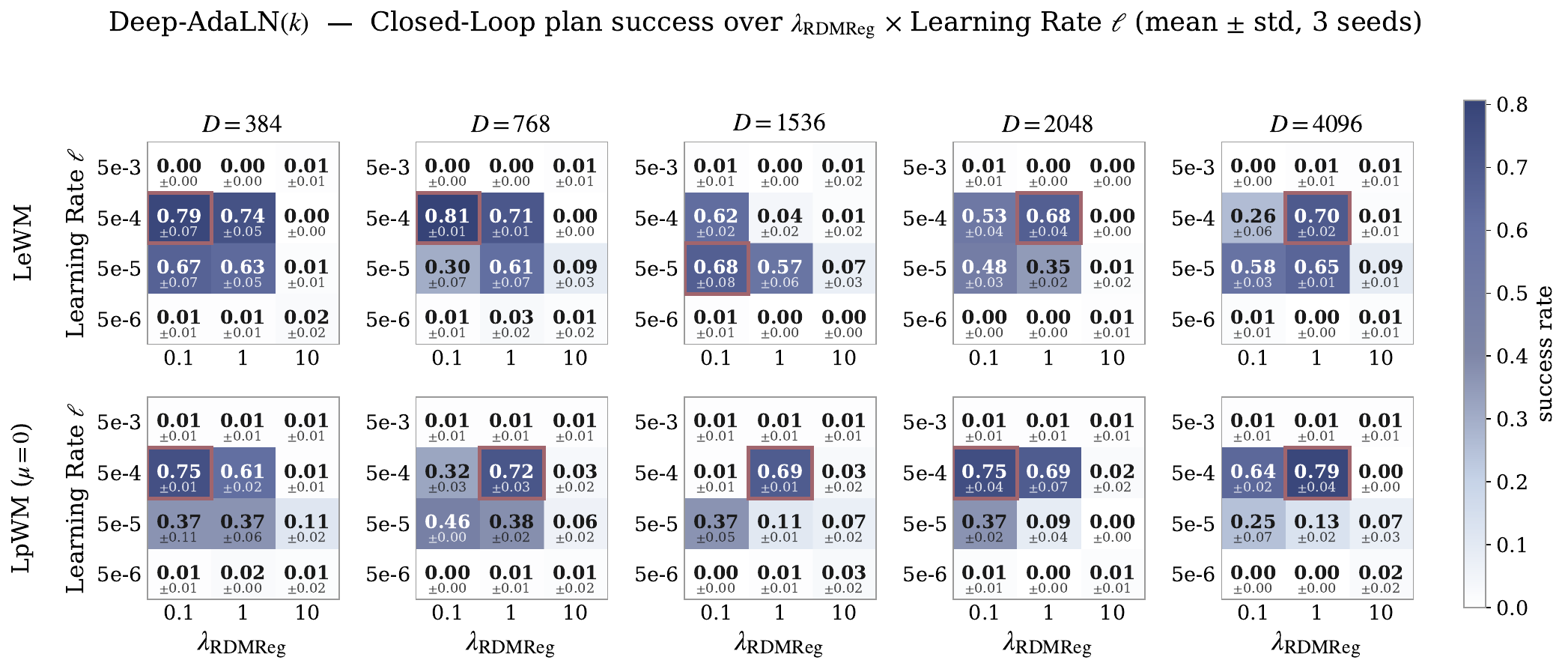}
        \captionsetup{justification=centering}
        \caption{Closed-Loop.}
        \label{fig:pushtsweepdeepadalnkclosed}
    \end{subfigure}

    \caption{\textbf{Deep-AdaLN$(k)$ Predictor.} Hyperparameter Sweeps for the PushT Environment.}
    \label{fig:pushtsweepdeepadalnk}
\end{figure*}

\begin{figure*}[t!]
    \centering

    \begin{subfigure}[t]{0.95\linewidth}
        \centering
        \includegraphics[width=\linewidth]{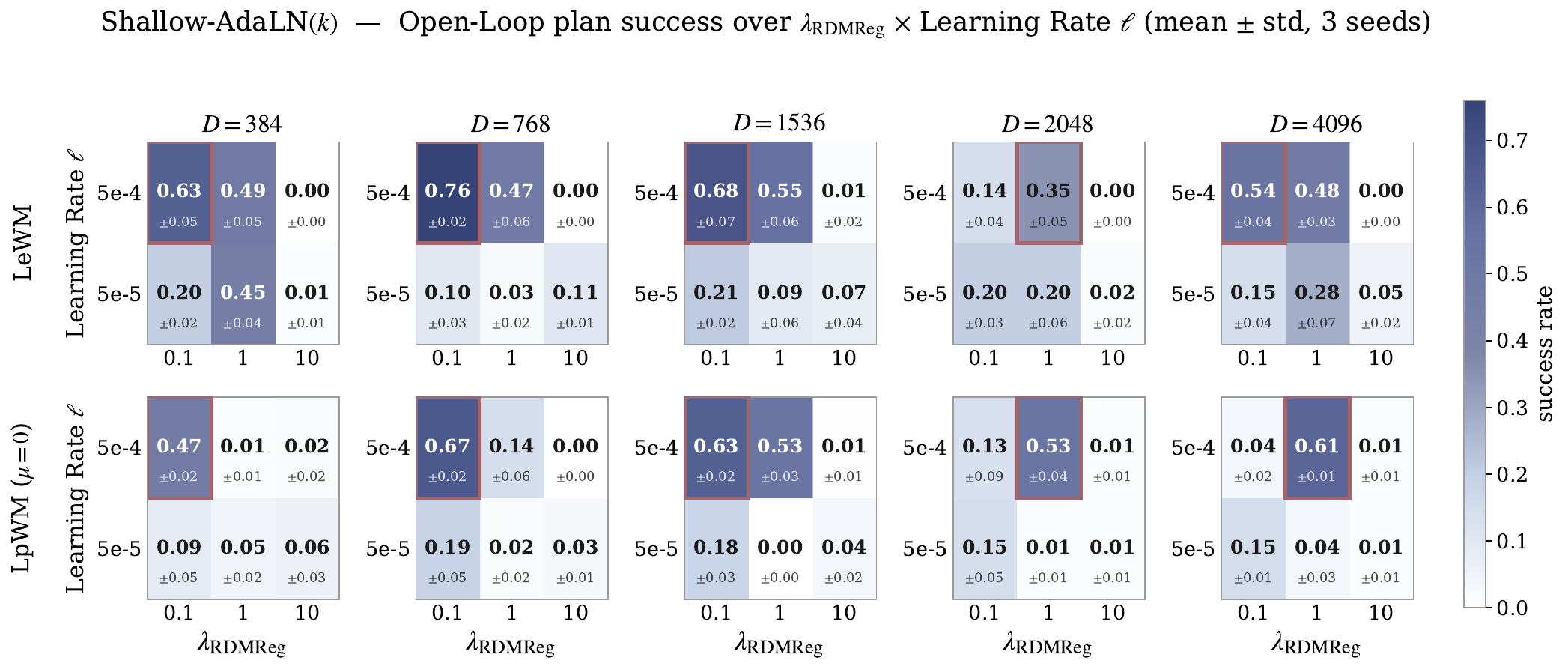}
        \captionsetup{justification=centering}
        \caption{Open-Loop.}
        \label{fig:pushtsweepshallowadalnkopen}
    \end{subfigure}

    \vspace{0.5cm}

    \begin{subfigure}[t]{0.95\linewidth}
        \centering
        \includegraphics[width=\linewidth]{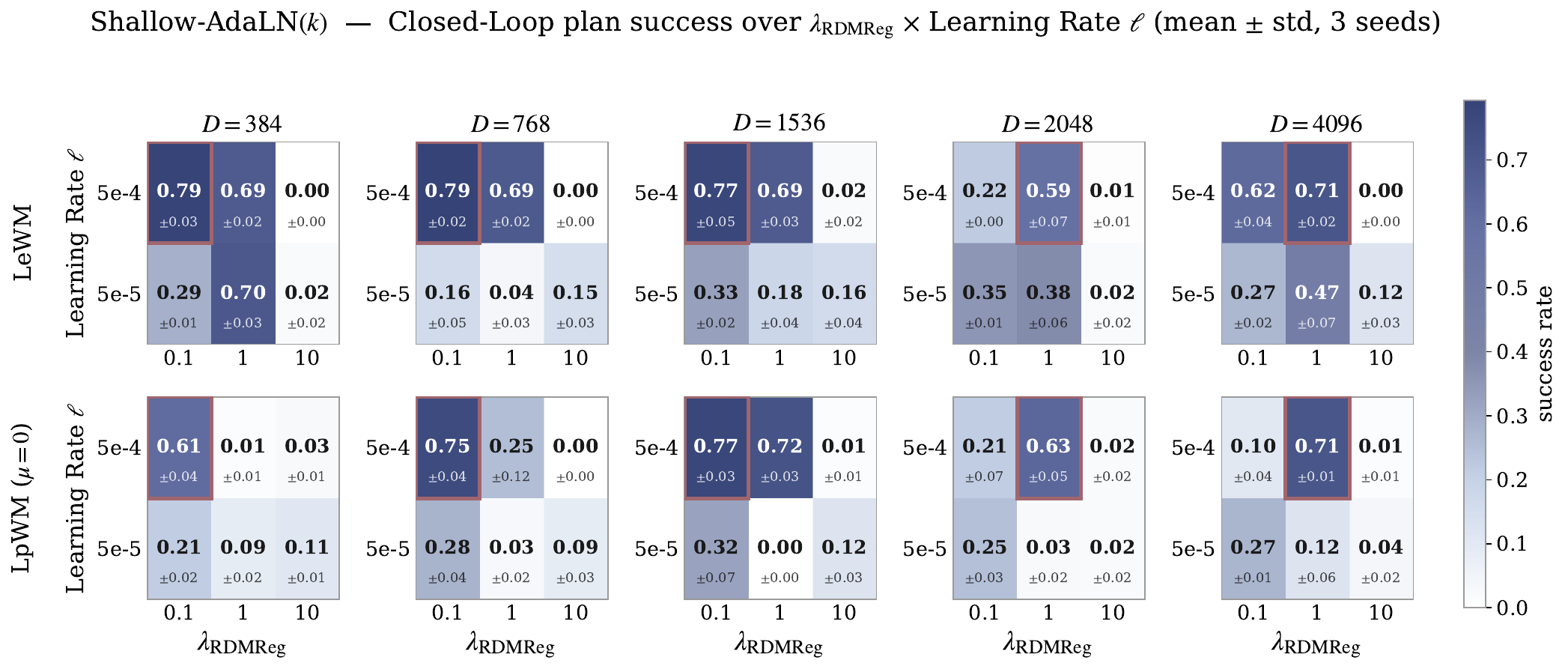}
        \captionsetup{justification=centering}
        \caption{Closed-Loop.}
        \label{fig:pushtsweepshallowadalnkclosed}
    \end{subfigure}

    \caption{\textbf{Shallow-AdaLN$(k)$ Predictor.} Hyperparameter Sweeps for the PushT Environment.}
    \label{fig:pushtsweepshallowadalnk}
\end{figure*}

\begin{figure*}[t!]
    \centering

    \begin{subfigure}[t]{0.95\linewidth}
        \centering
        \includegraphics[width=\linewidth]{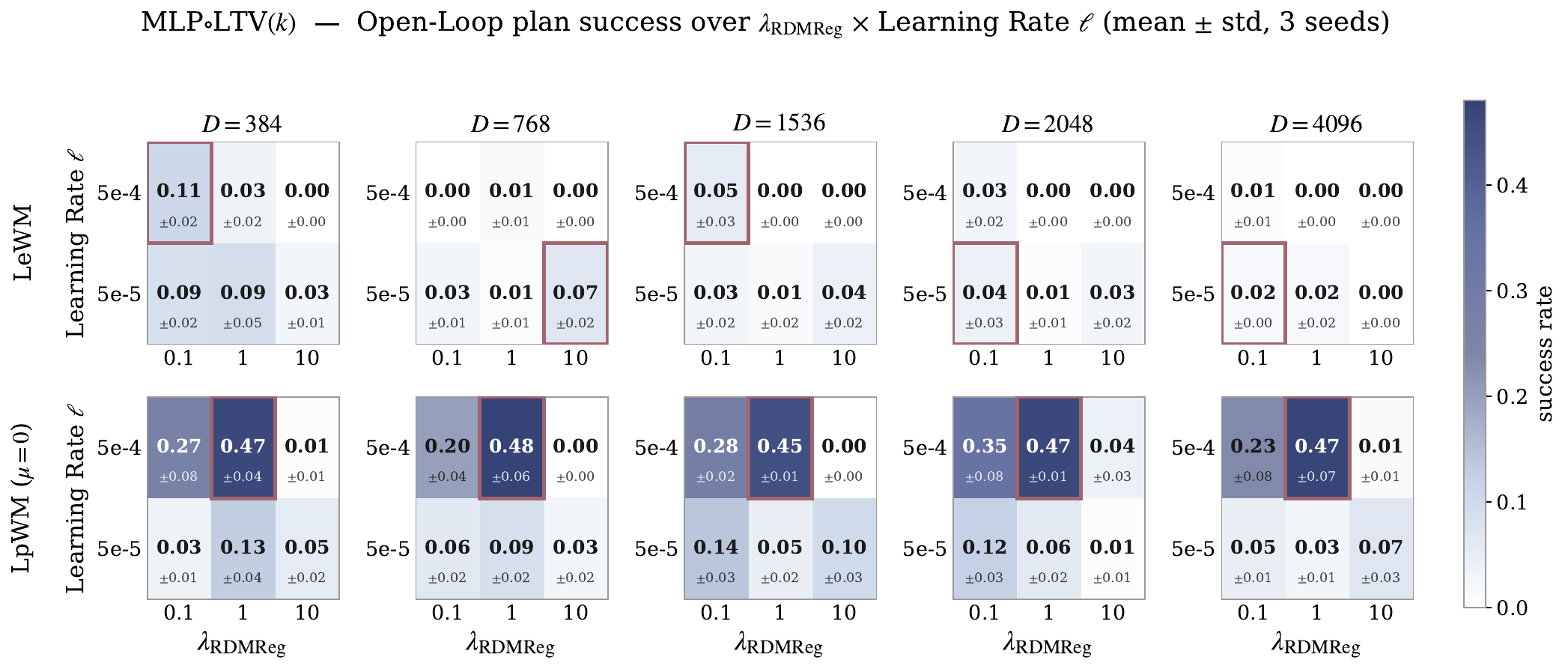}
        \captionsetup{justification=centering}
        \caption{Open-Loop.}
        \label{fig:pushtsweepmlpltvkopen}
    \end{subfigure}

    \vspace{0.5cm}

    \begin{subfigure}[t]{0.95\linewidth}
        \centering
        \includegraphics[width=\linewidth]{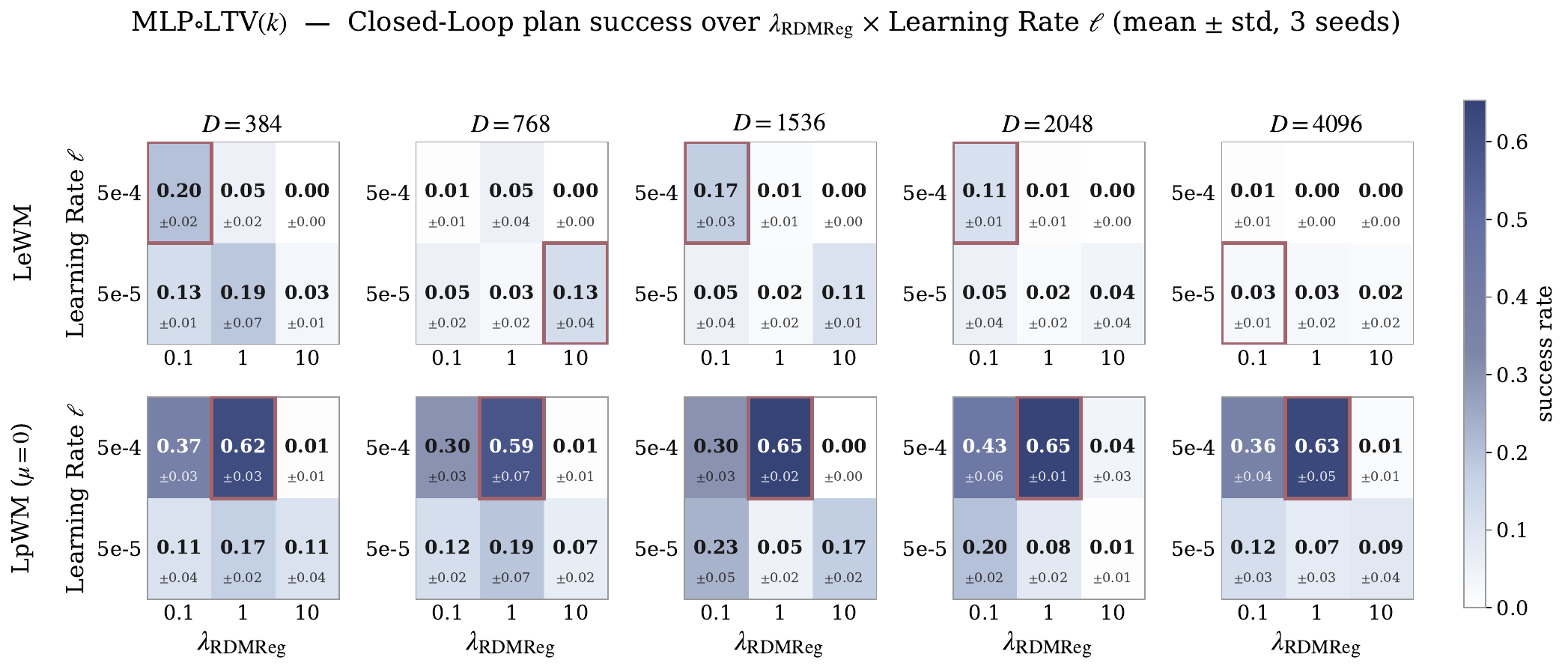}
        \captionsetup{justification=centering}
        \caption{Closed-Loop.}
        \label{fig:pushtsweepmlpltvkclosed}
    \end{subfigure}

    \caption{\textbf{MLP$\circ$LTV$(k)$ Predictor.} Hyperparameter Sweeps for the PushT Environment.}
    \label{fig:pushtsweepmlpltvk}
\end{figure*}

\begin{figure*}[t!]
    \centering

    \begin{subfigure}[t]{0.95\linewidth}
        \centering
        \includegraphics[width=\linewidth]{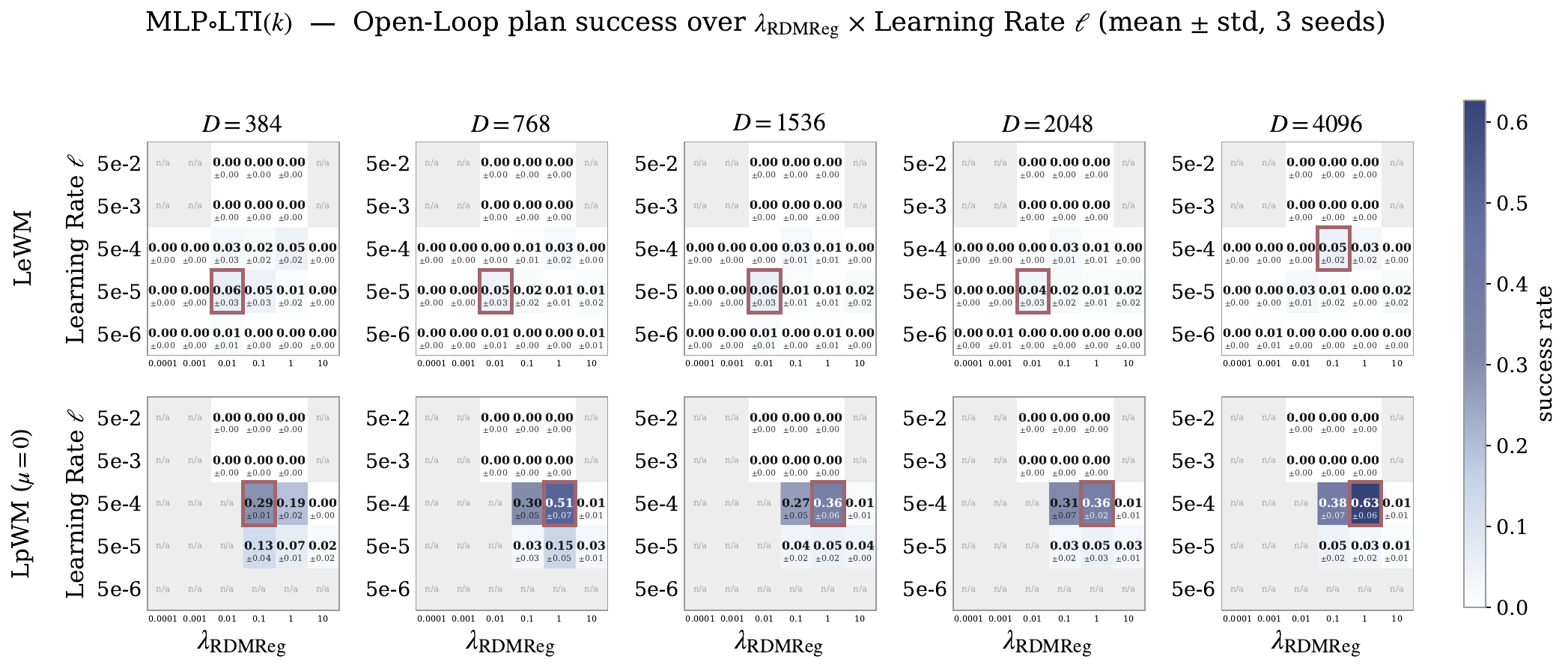}
        \captionsetup{justification=centering}
        \caption{Open-Loop.}
        \label{fig:pushtsweepmlpltikopen}
    \end{subfigure}

    \vspace{0.5cm}

    \begin{subfigure}[t]{0.95\linewidth}
        \centering
        \includegraphics[width=\linewidth]{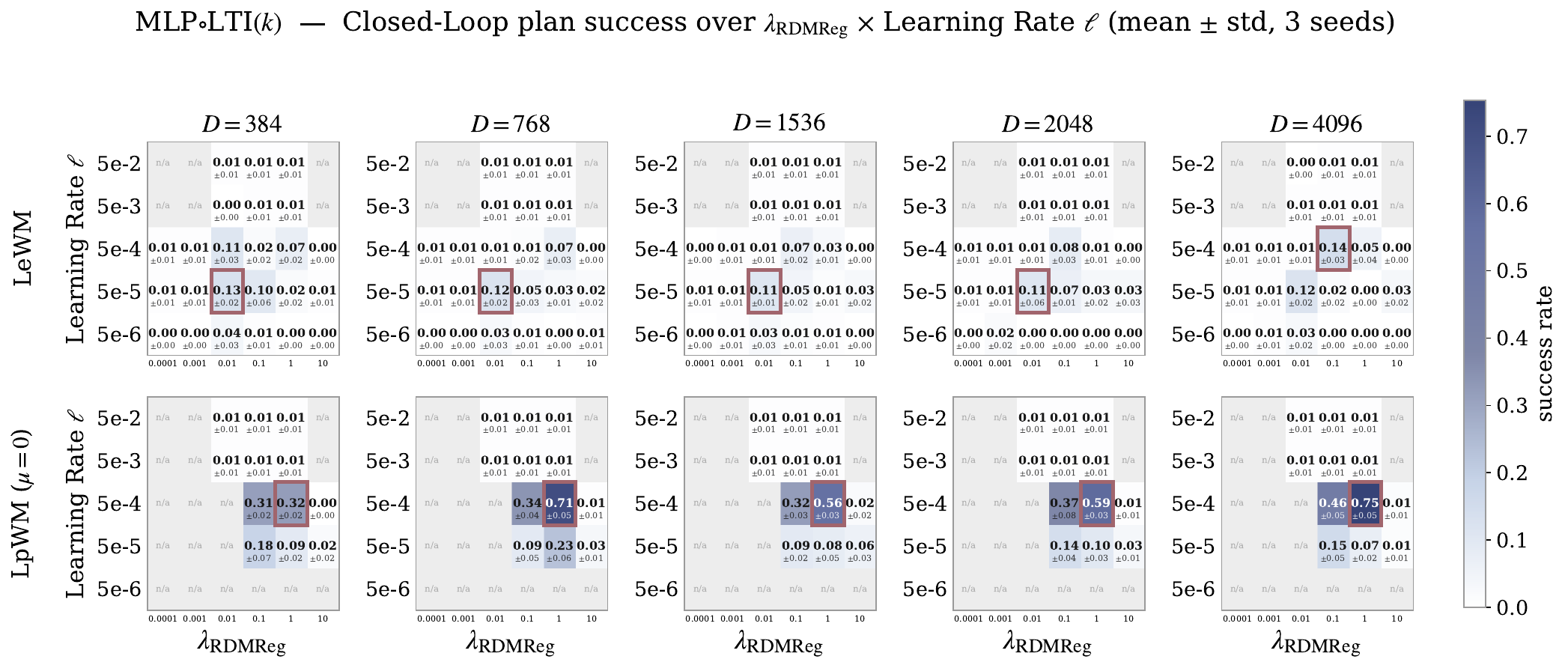}
        \captionsetup{justification=centering}
        \caption{Closed-Loop.}
        \label{fig:pushtsweepmlpltikclosed}
    \end{subfigure}

    \caption{\textbf{MLP$\circ$LTI$(k)$ Predictor.} Hyperparameter Sweeps for the PushT Environment.}
    \label{fig:pushtsweepmlpltik}
\end{figure*}

\begin{figure*}[t!]
    \centering

    \begin{subfigure}[t]{0.95\linewidth}
        \centering
        \includegraphics[width=\linewidth]{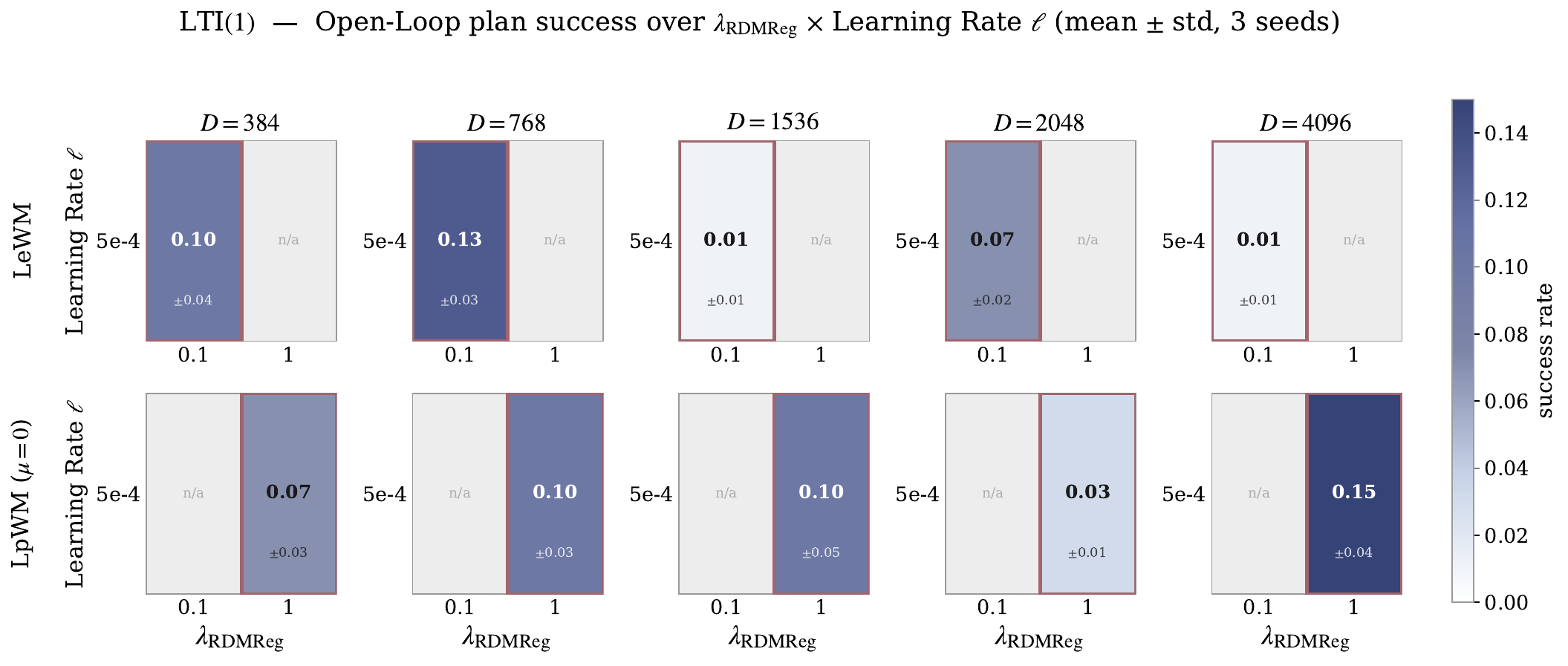}
        \captionsetup{justification=centering}
        \caption{Open-Loop.}
        \label{fig:pushtsweeplti1open}
    \end{subfigure}

    \vspace{0.5cm}

    \begin{subfigure}[t]{0.95\linewidth}
        \centering
        \includegraphics[width=\linewidth]{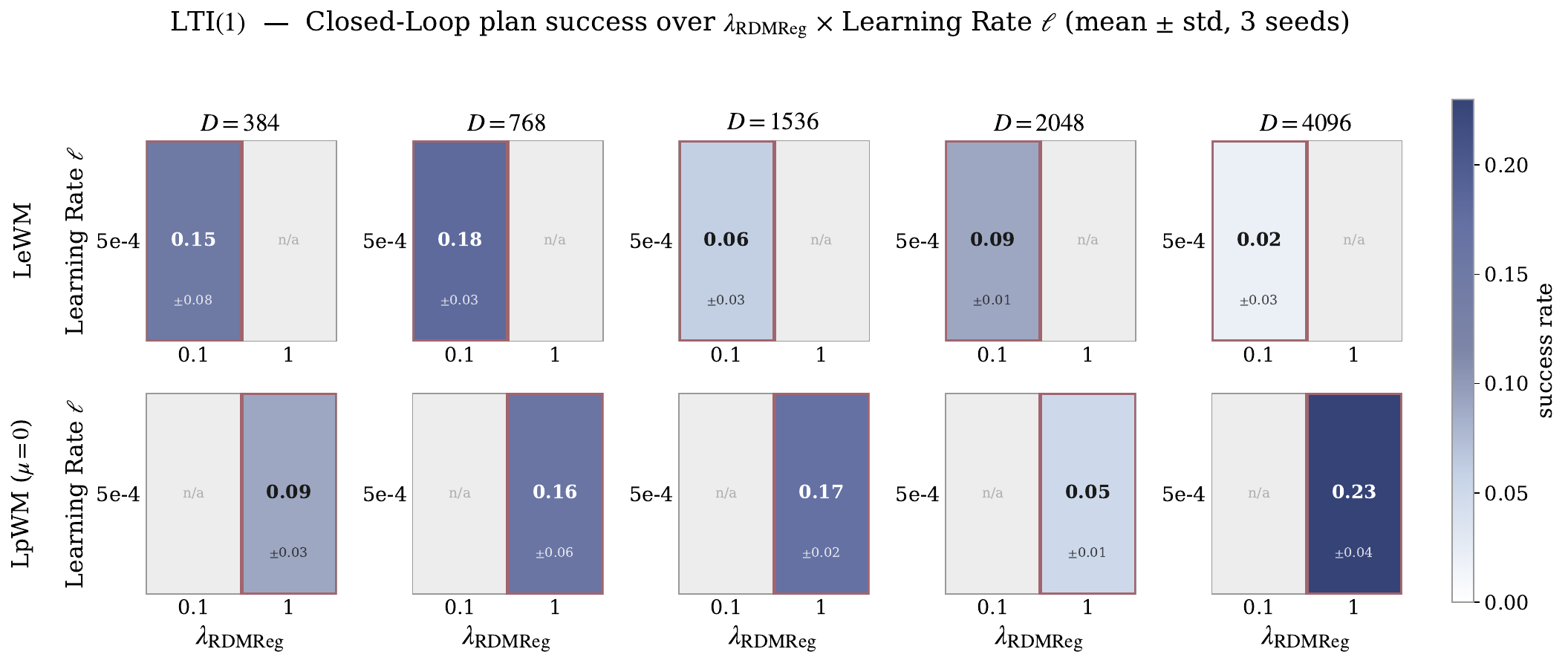}
        \captionsetup{justification=centering}
        \caption{Closed-Loop.}
        \label{fig:pushtsweeplti1closed}
    \end{subfigure}

    \caption{\textbf{LTI$(1)$ Predictor.} Hyperparameter Sweeps for the PushT Environment.}
    \label{fig:pushtsweeplti1}
\end{figure*}

\begin{figure*}[t!]
    \centering

    \begin{subfigure}[t]{0.95\linewidth}
        \centering
        \includegraphics[width=\linewidth]{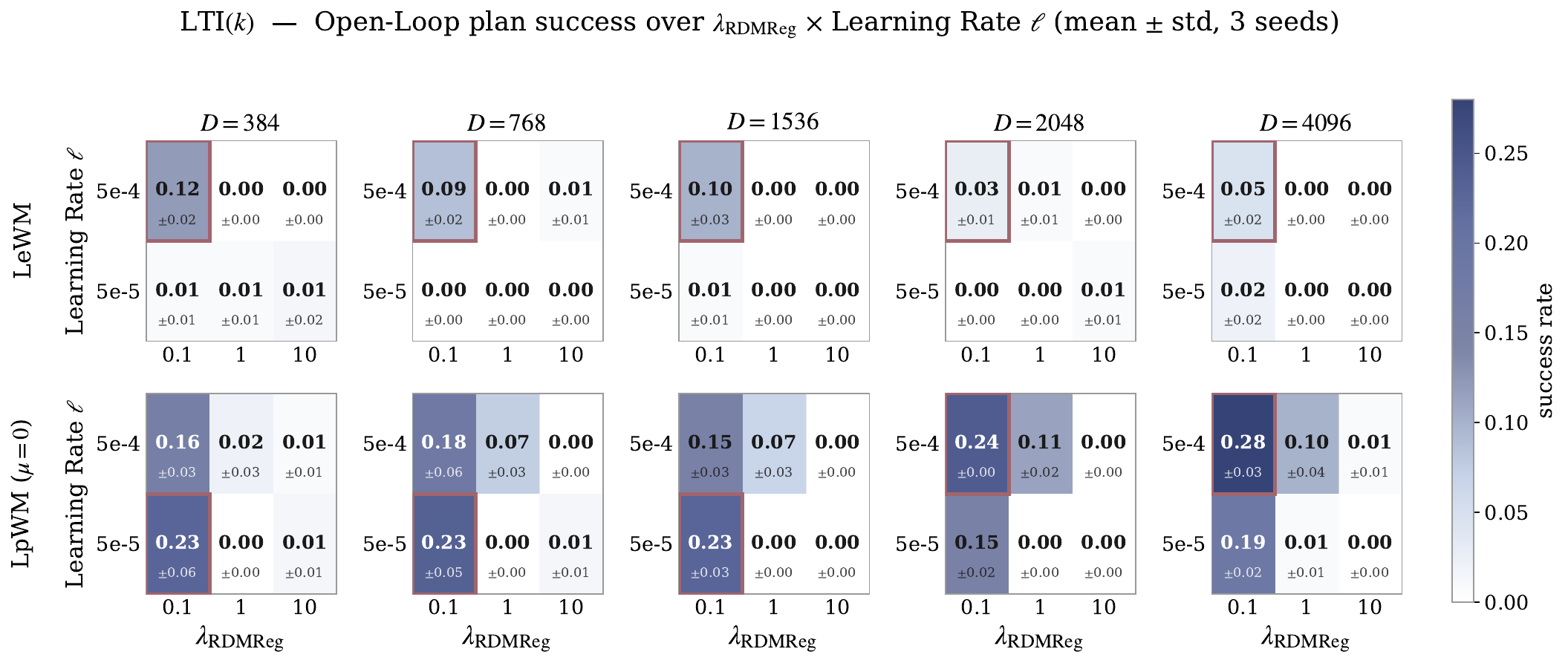}
        \captionsetup{justification=centering}
        \caption{Open-Loop.}
        \label{fig:pushtsweepltikopen}
    \end{subfigure}

    \vspace{0.5cm}

    \begin{subfigure}[t]{0.95\linewidth}
        \centering
        \includegraphics[width=\linewidth]{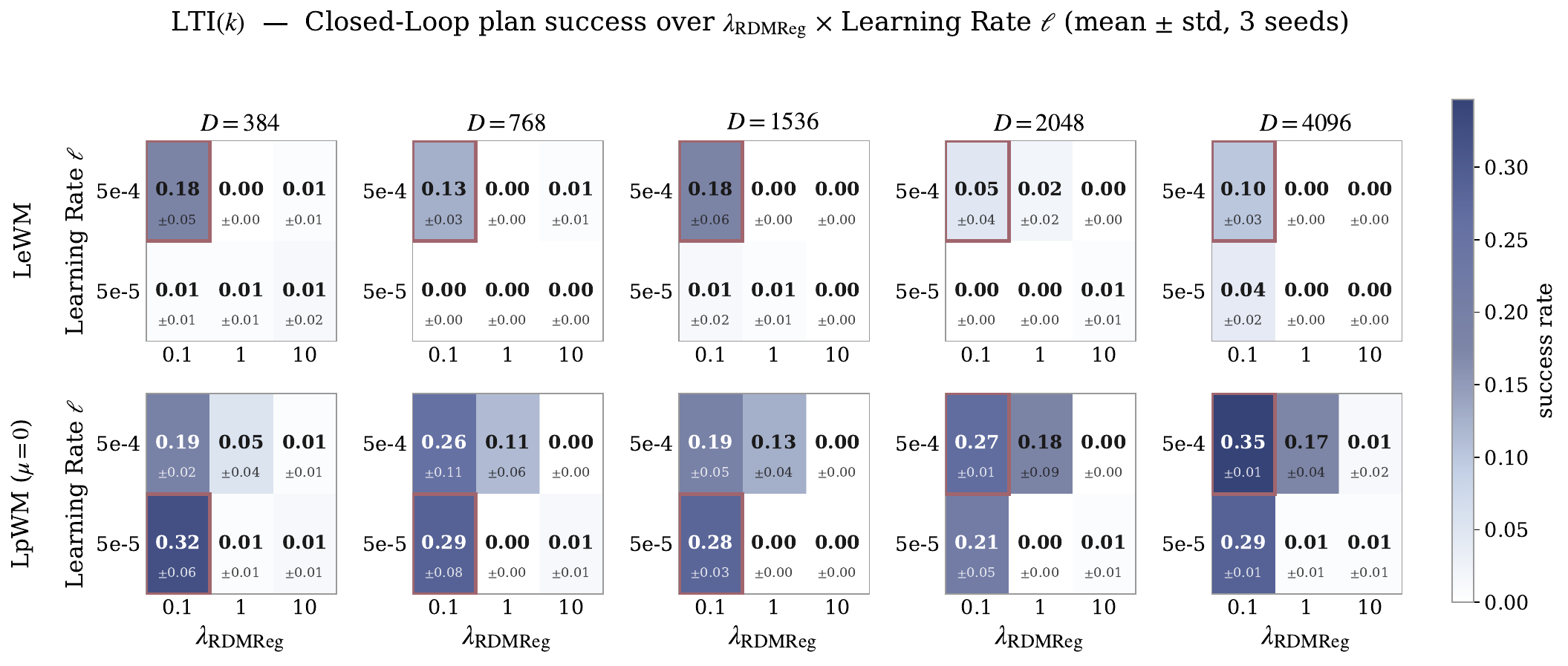}
        \captionsetup{justification=centering}
        \caption{Closed-Loop.}
        \label{fig:pushtsweepltikclosed}
    \end{subfigure}

    \caption{\textbf{LTI$(k)$ Predictor.} Hyperparameter Sweeps for the PushT Environment.}
    \label{fig:pushtsweepltik}
\end{figure*}

\begin{figure*}[t!]
    \centering

    \begin{subfigure}[t]{0.95\linewidth}
        \centering
        \includegraphics[width=\linewidth]{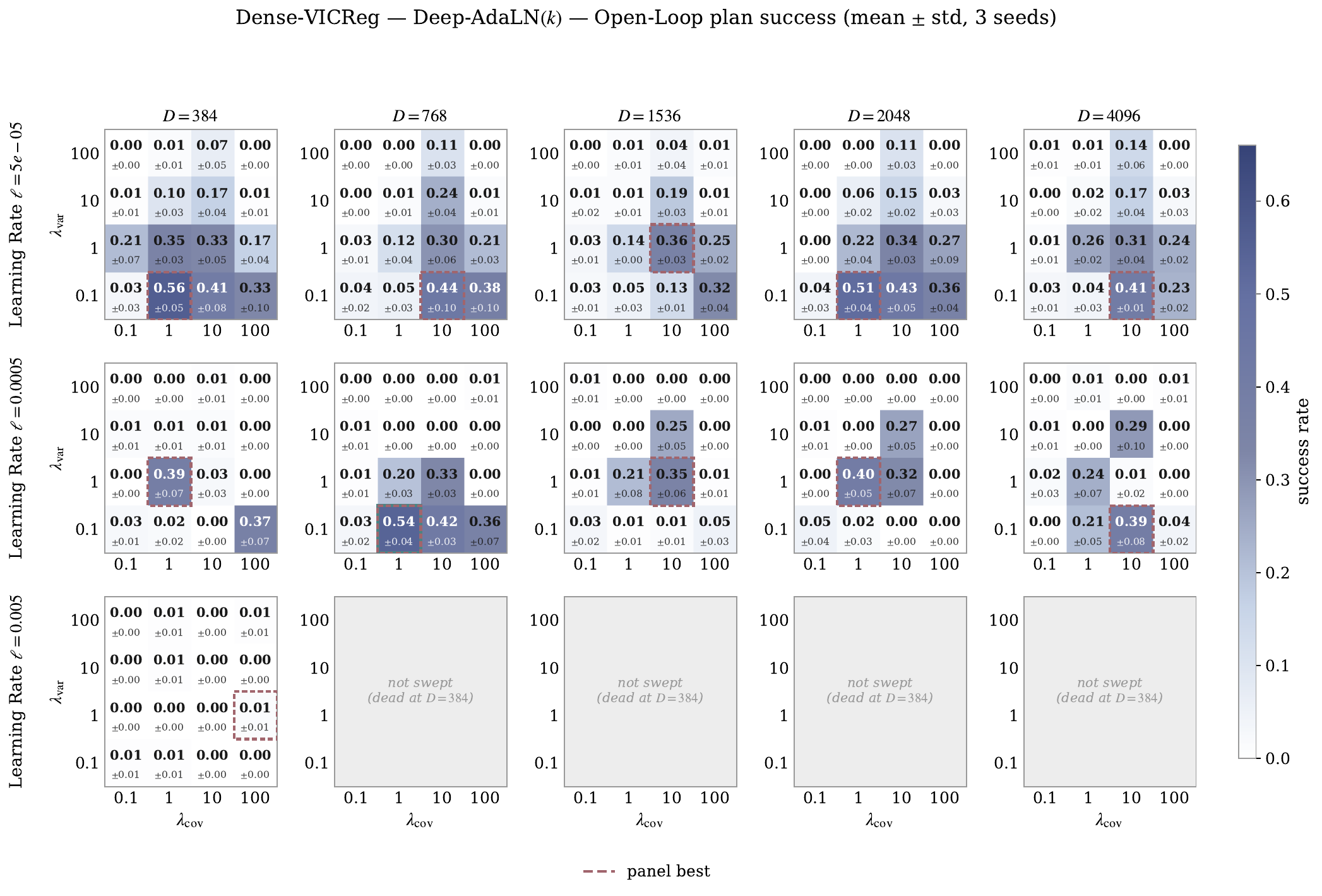}
        \captionsetup{justification=centering}
        \caption{Open-Loop.}
        \label{fig:vicpushtsweepdeepadalnkopen}
    \end{subfigure}

    \vspace{0.5cm}

    \begin{subfigure}[t]{0.95\linewidth}
        \centering
        \includegraphics[width=\linewidth]{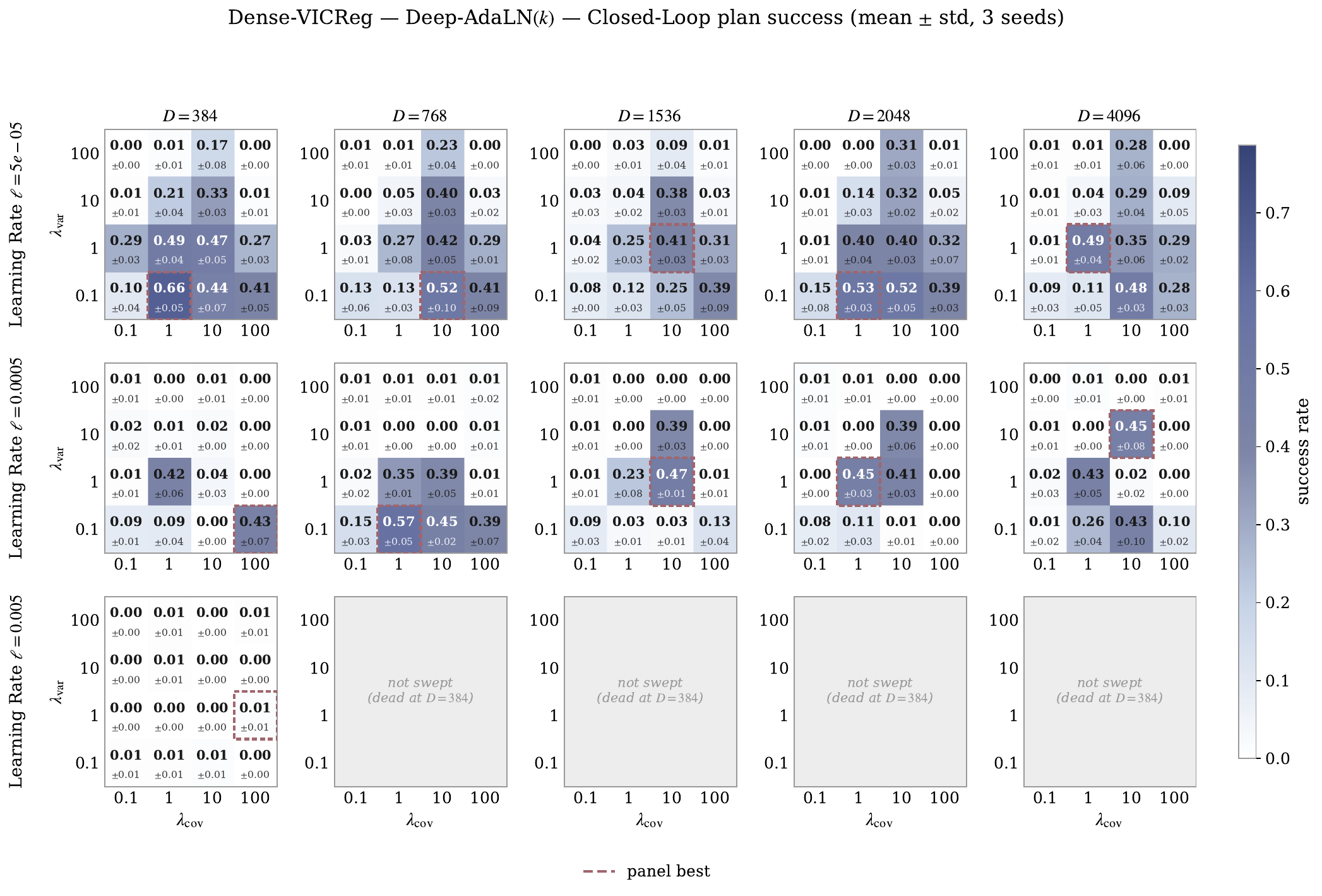}
        \captionsetup{justification=centering}
        \caption{Closed-Loop.}
        \label{fig:vicpushtsweepdeepadalnkclosed}
    \end{subfigure}

    \caption{\textbf{Deep-AdaLN$(k)$ Predictor.} Hyperparameter Sweeps for VICReg-based JEPA models in the PushT Environment.}
    \label{fig:vicpushtsweepdeepadalnk}
\end{figure*}

\begin{figure*}[t!]
    \centering

    \begin{subfigure}[t]{0.95\linewidth}
        \centering
        \includegraphics[width=\linewidth]{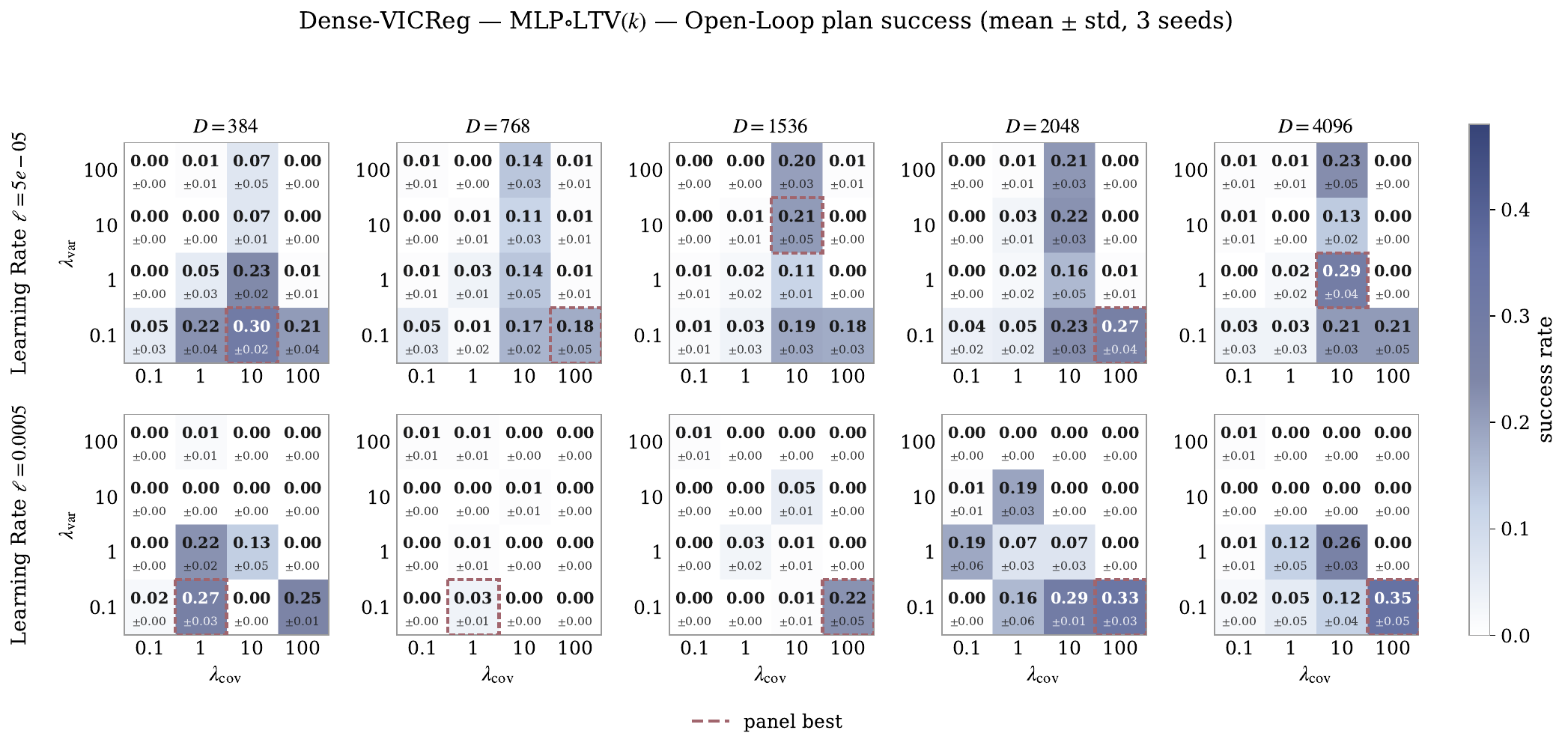}
        \captionsetup{justification=centering}
        \caption{Open-Loop.}
        \label{fig:vicpushtsweepmlpltvkopen}
    \end{subfigure}

    \vspace{0.5cm}

    \begin{subfigure}[t]{0.95\linewidth}
        \centering
        \includegraphics[width=\linewidth]{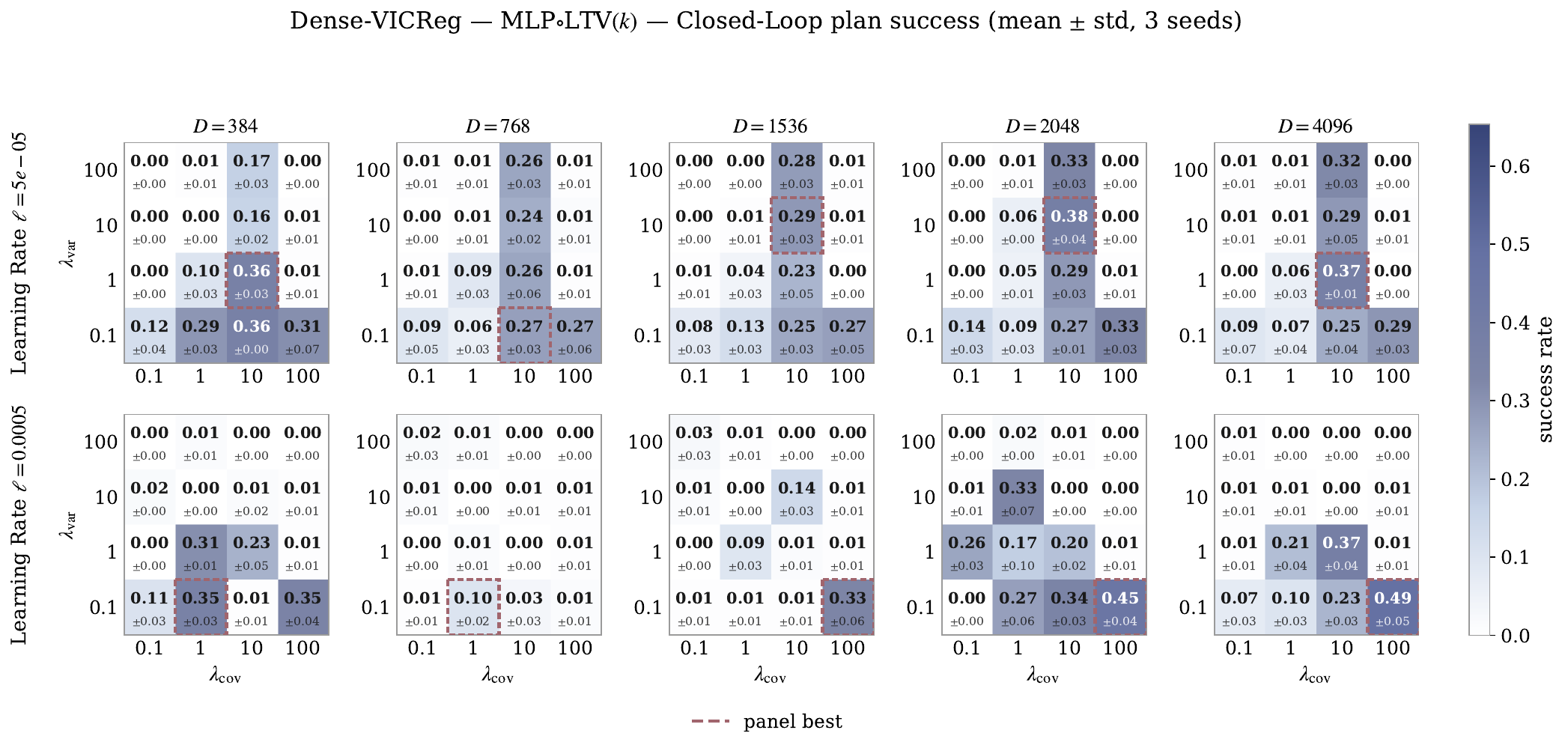}
        \captionsetup{justification=centering}
        \caption{Closed-Loop.}
        \label{fig:vicpushtsweepmlpltvkclosed}
    \end{subfigure}

    \caption{\textbf{MLP$\circ$LTV$(k)$ Predictor.} Hyperparameter Sweeps for VICReg-based JEPA models in the PushT Environment.}
    \label{fig:vicpushtsweepmlpltvk}
\end{figure*}

\begin{figure*}[t!]
    \centering

    \begin{subfigure}[t]{0.95\linewidth}
        \centering
        \includegraphics[width=\linewidth]{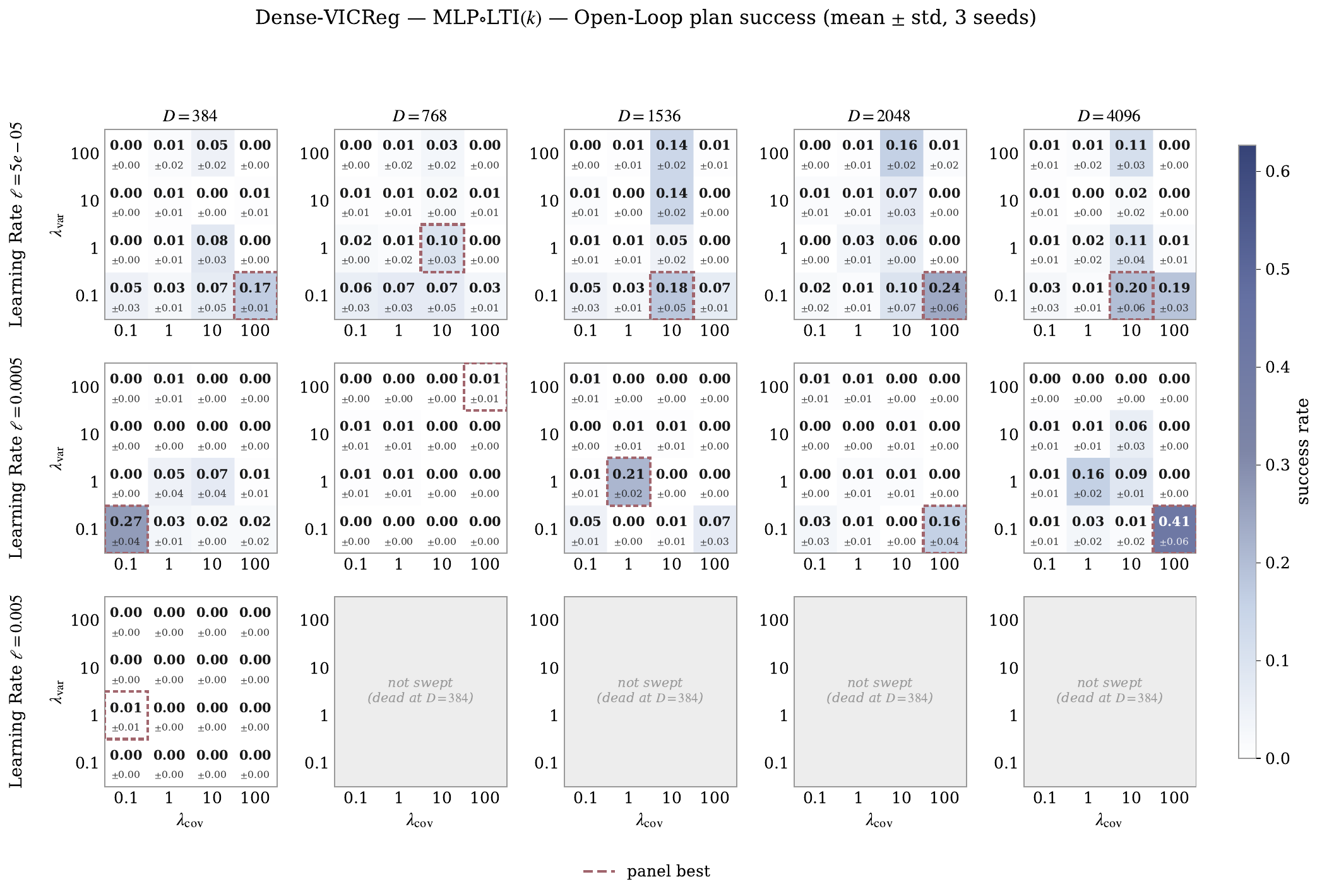}
        \captionsetup{justification=centering}
        \caption{Open-Loop.}
        \label{fig:vicpushtsweepmlpltikopen}
    \end{subfigure}

    \vspace{0.5cm}

    \begin{subfigure}[t]{0.95\linewidth}
        \centering
        \includegraphics[width=\linewidth]{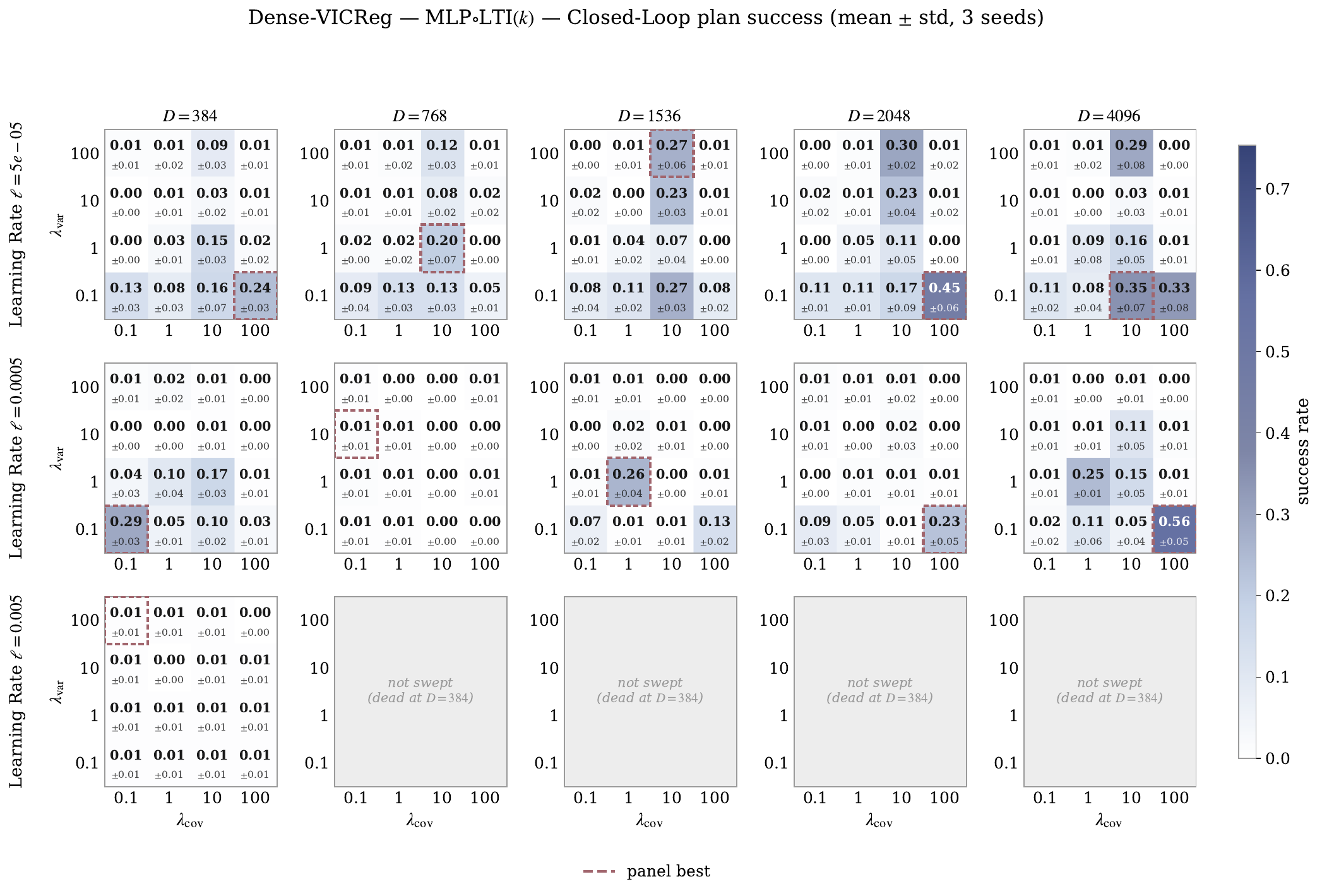}
        \captionsetup{justification=centering}
        \caption{Closed-Loop.}
        \label{fig:vicpushtsweepmlpltikclosed}
    \end{subfigure}

    \caption{\textbf{MLP$\circ$LTI$(k)$ Predictor.} Hyperparameter Sweeps for VICReg-based JEPA models in the PushT Environment.}
    \label{fig:vicpushtsweepmlpltik}
\end{figure*}

\end{document}